%% file: CounteractiveReinforcementLearning_NeurIPS2025Spotlight.tex
\documentclass{article}

\input{math_commands.tex}

    \usepackage[final]{neurips_2025}

\usepackage[utf8]{inputenc} % allow utf-8 input
\usepackage[T1]{fontenc}    % use 8-bit T1 fonts
\usepackage{hyperref}       % hyperlinks
\usepackage{url}            % simple URL typesetting
\usepackage{booktabs}       % professional-quality tables
\usepackage{amsfonts}       % blackboard math symbols
\usepackage{nicefrac}       % compact symbols for 1/2, etc.
\usepackage{microtype}      % microtypography
\usepackage{xcolor}         % colors

\usepackage{algorithmic}
\usepackage{algorithm}

\usepackage{amsmath}
\usepackage{amssymb}
\usepackage{mathtools}
\usepackage{amsthm}
\usepackage{varwidth}

\theoremstyle{plain}
\newtheorem{theorem}{Theorem}[section]
\newtheorem{proposition}[theorem]{Proposition}

\theoremstyle{definition}
\newtheorem{definition}[theorem]{Definition}

\theoremstyle{remark}

\usepackage[utf8]{inputenc} % allow utf-8 input
\usepackage[T1]{fontenc}    % use 8-bit T1 fonts
\usepackage{hyperref}       % hyperlinks
\usepackage{url}            % simple URL typesetting
\usepackage{booktabs}       % professional-quality tables
\usepackage{amsfonts}       % blackboard math symbols
\usepackage{nicefrac}       % compact symbols for 1/2, etc.
\usepackage{microtype}      % microtypography
\usepackage{xcolor}         % colors

\usepackage{amsmath}
\usepackage{amssymb}
\usepackage{mathtools}
\usepackage{amsthm}
\usepackage{vwcol}

\usepackage{multicol}
\usepackage{tikz} % nice language for creating drawings and diagrams
\usepackage{amsmath,amssymb,amsthm,textcomp}
\usepackage{comment}
\usepackage{wrapfig}
\usepackage{stackengine}
\usepackage{pifont}% http://ctan.org

\newlength{\RoundedBoxWidth}
\newsavebox{\GrayRoundedBox}
\newenvironment{GrayBox}[1][\dimexpr\textwidth-4.5ex]%
{\setlength{\RoundedBoxWidth}{\dimexpr#1}
\begin{lrbox}{\GrayRoundedBox}
\begin{minipage}{\RoundedBoxWidth}}%
{   \end{minipage}
\end{lrbox}
\begin{center}
\begin{tikzpicture}%
\draw node[draw=black,fill=black!7,rounded corners,%
inner sep=2ex,text width=\RoundedBoxWidth]%
{\usebox{\GrayRoundedBox}};
\end{tikzpicture}
\end{center}}

\usepackage[textsize=tiny]{todonotes}

\title{Counteractive RL: Rethinking Core Principles for Efficient and Scalable Deep Reinforcement Learning}

\author{
Ezgi Korkmaz
\thanks{Correspondence to Ezgi Korkmaz: \texttt{ezgikorkmazmail@gmail.com} }
}

\begin{document}
\maketitle

\begin{abstract}
Following the pivotal success of learning strategies to win at tasks, solely by interacting with an environment without any supervision, agents have gained the ability to make sequential decisions in complex MDPs. Yet, reinforcement learning policies face exponentially growing state spaces in high dimensional MDPs resulting in a dichotomy between computational complexity and policy success. In our paper we focus on the agent’s interaction with the environment in a high-dimensional MDP during the learning phase and we introduce a theoretically-founded novel paradigm based on experiences obtained through counteractive actions. 
Our analysis and method provide a theoretical basis for efficient, effective, scalable and accelerated learning, and further comes with zero additional computational complexity while leading to significant acceleration in training.
We conduct extensive experiments in the Arcade Learning Environment with high-dimensional state representation MDPs. The experimental results further verify our theoretical analysis, and our method achieves significant performance increase with substantial sample-efficiency in high-dimensional environments.
\end{abstract}

\section{Introduction}

Utilization of deep neural networks as function approximators enabled learning functioning policies in high-dimensional state representation MDPs \citep{mn15}. Following this initial work, the current line of research trains deep reinforcement learning policies to solve highly complex problems from game solving \citep{hado16,julian20} to mathematical and scientific reasoning of large language models \citep{deepseek25}.
Yet, there are still remaining unsolved problems restricting the current capabilities of reinforcement learning in exponentially growing state spaces.
One of the main intrinsic open problems in deep reinforcement learning research is sample complexity and experience collection in high-dimensional state representation MDPs. While prior work extensively studied the policy's interaction with the environment
in bandits and tabular reinforcement learning, and proposed various algorithms and techniques optimal to the tabular form or the bandit context \citep{claude94,kearn02, ronen02, kakade03, roy19}, experience collection in deep reinforcement learning remains an open challenging problem while practitioners repeatedly employ quite simple yet effective techniques (i.e. $\epsilon$-greedy) \citep{white91,flen22,hado16, wang16,jess20,hado23, korkmaz24icml, thomas25, akshay24}.

Despite the provable optimality of the techniques designed for the tabular or bandit setting, they generally rely strongly on the assumptions of tabular reinforcement learning, and in particular on the ability to record tables of statistical estimates for every state-action pair which have size growing with the number of states times the number of actions.
Hence, these assumptions are far from what is being faced in the deep reinforcement learning setting where states and actions can be parametrized by high-dimensional representations.
Thus, in high-dimensional complex MDPs, for which deep neural networks are used as function approximators, the efficiency and the optimality of the methods proposed for tabular settings do not transfer well to deep reinforcement learning \citep{kakade03}.
Hence, in deep reinforcement learning research still, naive and standard techniques (e.g. $\epsilon$-greedy) are preferred over both the optimal tabular techniques and over the particular recent experience collection techniques targeting only high scores for particular games \citep{mn15, hado16, wang16, bell17, dabney18, flen22, korkmaz24icml, hado23}.

Sample efficiency still remains to be one of the main challenging problems restricting research progress in reinforcement learning. The magnitude of the number of samples required to learn and adapt continuously is one of the main limiting factors preventing current state-of-the-art deep reinforcement learning algorithms from being deployed in many diverse settings from large language model reasoning to the physical world, but most importantly one of the main challenges that needs to be dealt with on the way to building neural policies that can generalize and adapt continuously in non-stationary environments.
Hence, given these limitations in our paper we aim to seek answers for the following questions:
\begin{itemize}
\item \textit{How can we construct policies that have the ability to collect novel experiences in high-dimensional complex MDPs without any additional computational complexity?}
\item \textit{What is the natural theoretical motivation that can be used to design a zero-cost experience collection strategy while achieving high sample efficiency?}
\end{itemize}

To be able to answer these questions, in our paper we focus on environment interactions in deep reinforcement learning and make the following contributions:
\paragraph{Contributions.}
We introduce a fundamental theoretically well-motivated paradigm for reinforcement learning
based on state-action value function minimization, which we call counteractive temporal difference learning. Our approach centers on solely reconstituting and conceptually shifting the core principles of learning and as a result increases the information gained from the environment interactions of the policy in a given MDP without adding computational complexity.
We first provide the theoretical analysis in Section \ref{worstq}, explaining why and how minimization will result in higher temporal difference.
We then as a first step demonstrate the efficacy of counteractive temporal difference learning in a motivating example, i.e. the canonical chain MDP setup, in Section \ref{motivatingexample}. The results in the chain MDP verify the theoretical analysis provided in Section \ref{worstq} that counteractive temporal difference learning increases temporal difference obtained from the experiences.
Furthermore, we conduct an extensive study in the Arcade Learning Environment 100K benchmark with the state-of-the-art algorithms and demonstrate that our temporal difference learning algorithm CoAct TD learning improves performance by $248\%$ across the entire benchmark compared to the baseline algorithm.
We demonstrate the efficacy of our proposed CoAct TD Learning algorithm in terms of sample-efficiency. Our method based on maximizing novel experiences via minimizing the state-action value function reaches approximately to the same performance level as model-based deep reinforcement learning algorithms, without building and learning any model of the environment.
Finally, we show that CoAct TD learning is a fundamental improvement over canonical methods, it is modular and a plug-and-play method, and any algorithm that uses temporal difference learning can be immediately and simply switched to CoAct TD learning.

\section{Background and Preliminaries}

\label{exploration}
The reinforcement learning problem is formalized as a Markov Decision Process (MDP) \citep{puterman94} $\mathcal{M} =\langle \mathcal{S}, \mathcal{A}, r, \gamma, \rho_0, \mathcal{T} \rangle $ that contains a continuous set of states $s \in \mathcal{S}$, a set of actions $a \in \mathcal{A}$, a probability transition function $\mathcal{T}(s,a,s')$ on $\mathcal{S}\times \mathcal{A} \times \mathcal{S}$, discount factor $\gamma$,
a reward function $r(s,a): \mathcal{S} \times \mathcal{A} \to \R$ with initial state distribution $\rho_0$.
A policy $\pi(s,a):  \mathcal{S} \times \mathcal{A} \to [0,1]$ in an MDP assigns a probability distribution over actions for each state $s \in \mathcal{S}$. The main goal in reinforcement learning is to learn an optimal policy $\pi$ that maximizes the discounted expected cumulative rewards $\mathcal{R} = \mathbb{E}_{a_t \sim \pi(s_t,\cdot),s_{t+1} \sim \mathcal{T}(s_t,a_t,\cdot)} \sum_t \gamma^t r(s_t,a_t)$.
In $Q$-learning \citep{watkins89} the learned policy is parameterized by a state-action value function $Q:\mathcal{S}\times \mathcal{A}\to \R$, which represents the value of taking action $a$ in state $s$. The optimal state-action value function is learnt via iterative Bellman update
\[
Q(s_t,a_t) \leftarrow Q(s_t,a_t) + \alpha [ r(s_t, a_t) + \gamma \max_a Q(s_{t+1}, a) - Q(s_t, a_t)]
\]
where $\max_a Q(s_{t+1},a) = \mathcal{V}(s_{t+1})$.
Let $a^*$ be the action maximizing the state-action value function, $a^*(s) = \argmax_a Q(s,a)$, in state $s$. Once the $Q$-function is learnt the policy is determined via taking action $a^*(s)$.
Temporal difference learning \citep{sutton88} improves the estimates of the state-action values in each iteration via the Bellman Operator \citep{bellman57}
\begin{align*}
(\Omega^* Q)(s,a) =  \mathbb{E}_{s' \sim \mathcal{T}(s,a,\cdot)} [ r(s,a) + \gamma  \max_{a'} Q(s',a')].
\end{align*}
For distributional reinforcement learning, QRDQN is an algorithm that is based on quantile regression \citep{koenker01, koenker05} temporal difference learning
\begin{align*}
\Omega \mathcal{Z}(s,a) = r(s,a) + \gamma \mathcal{Z}(s',\argmax_{a'} \mathbb{E}_{z \sim \mathcal{Z}(s',a')}[z]) \:\:
 \textrm{and} \:\: \mathcal{Z}(s,a) \coloneqq \dfrac{1}{N} \sum_{i=1}^N \delta_{\theta_i(s,a)} 
\end{align*}
where $\mathcal{Z}_\theta \in \mathcal{Z}_Q$ maps state-action pairs to a probability distribution over values.
In deep reinforcement learning, the state space or the action space is large enough that it is not possible to learn and store the state-action values in a tabular form. Thus, the $Q$-function is approximated via deep neural networks. 
In deep double-$Q$ learning, two $Q$-networks are used to decouple the $Q$-network deciding which action to take and the $Q$-network to evaluate the action taken 
\[
\theta_{t+1} = \theta_{t} + \alpha (r(s_t,a_t)
+ \gamma Q(s_{t+1},  \argmax_a Q(s_{t+1},a;\theta_t);\hat{\theta}_t)
- Q(s_t,a_t;\theta_t)) \nabla_{\theta_t} Q(s_t,a_t;\theta_t).
\]
Current deep reinforcement learning algorithms use $\epsilon$-greedy during training \citep{wang16, mn15, hado16, jess20,flen22,hado23, akshay24, thomas25, korkmazaaai26}. In particular, the $\epsilon$-greedy \citep{white91} algorithm takes an action $a_k \sim \mathcal{U}(\mathcal{A})$ with probability $\epsilon$ in a given state $s$, i.e. $\pi(s,a_k)=\frac{\epsilon}{\lvert\mathcal{A}\rvert}$, and takes an action $a^* = \argmax_a Q(s,a)$ with probability $1-\epsilon$, i.e.
\[
\pi(s,\argmax_a Q(s,a))=1-\epsilon+\frac{\epsilon}{\lvert\mathcal{A}\rvert}
\]
While a family of algorithms have been proposed based on counting state visitations (i.e. the number of times action $a$ has been taken in state $s$ by time step $t$) with provable optimal regret bounds using the principal of optimism in the face of uncertainty in the tabular MDP setting, yet incorporating these count-based methods in high-dimensional state representation MDPs requires substantial complexity including training additional deep neural networks to estimate counts or other uncertainty metrics.
As a result, many state-of-the-art deep reinforcement learning algorithms still use simple, randomized experience collection methods based on sampling a uniformly random action with probability $\epsilon$ \citep{mn15,hado16, wang16, jess20,flen22,korkmazaaai23,hado23}.
In our experiments, while providing comparison against canonical methods, we also compare our method against computationally complicated and expensive
techniques such as noisy-networks that is based on the injection of random noise with additional layers in the deep neural network \citep{hessel2018rainbow} in Section \ref{largeexp}, and count based methods in Section \ref{motivatingexample} and Section \ref{temporal}.
We further highlight that our method is a fundamental theoretically motivated improvement of temporal difference learning. 
Thus, any algorithm that is based on temporal difference learning can immediately be switched to CoAct TD learning.

\section{Maximizing Temporal Difference with Counteractive Actions}
\label{worstq}

Seeking experiences that contain high information has long been the focus of reinforcement learning \citep{schmid91,jurgen99, moore93} and more particularly the experiences that correspond to higher temporal difference \citep{moore93}.
In this section we will provide the theoretical analysis for our proposed algorithm counteractive TD learning. 
Section \ref{largeexp} further provides the experimental results verifying the theoretical predictions. 
In deep reinforcement learning the state-action value function is initialized with random weights \citep{mn15, mn16, hado16, wang16, tom16,oh20, julian20, hubert21}. 
During a large portion of the training prior to convergence, the $Q$-function behaves as a random function rather than providing an accurate representation of the optimal state-action values while interacting with new experiences in high-dimensional MDPs as the learning continues.
In particular, in high-dimensional environments in a significant portion of the training the $Q$-function, on average, assigns approximately similar values to states that are similar, and has little correlation with the immediate rewards. Hence, let us formalize these facts on the state-action value function in the following definitions.

\begin{definition}[\emph{$\eta$-uninformed}]
\label{etauninformed}
Let $\eta > 0$. A $Q$-function parameterized by weights $\theta \sim \Theta$ is $\eta$-uninformed if for any state $s \in \mathcal{S}$ with $a^{\textrm{min}} = \argmin_a Q_\theta(s,a)$ we have
\begin{align*}
\lvert\mathbb{E}_{\theta \sim \Theta}[r(s_t,a^{\textrm{min}})] - \mathbb{E}_{a\sim \mathcal{U}(\mathcal{A})}[r(s_t,a)]\rvert < \eta.
\end{align*}
\end{definition}

\begin{definition}[\emph{$\delta$-smooth}]
\label{tderrdef}
Let $\delta > 0$. A $Q$-function parameterized by weights $\theta \sim \Theta$ is $\delta$-smooth if for any state $s \in \mathcal{S}$ and action $\hat{a} = \hat{a}(s,\theta)$ with $s' \sim \mathcal{T}(s,\hat{a},\cdot)$ we have
\begin{align*}
\lvert\mathbb{E}_{\theta \sim \Theta}[\max_a Q_\theta(s,a)]  
- \mathbb{E}_{s' \sim \mathcal{T}(s,\hat{a},\cdot), \theta \sim \Theta}[\max_a Q_\theta(s',a)]\rvert < \delta
\end{align*}
where the expectation is over both the random initialization of the $Q$-function weights, and the random transition to state $s' \sim \mathcal{T}(s,\hat{a},\cdot)$.
\end{definition}

\begin{definition}[\emph{Disadvantage Gap}]
\label{disadvantagedef}
For a state-action value function $Q_{\theta}$ the disadvantage gap in a state $s \in \mathcal{S}$ is given by $\mathcal{D}(s) = \mathbb{E}_{a \sim \mathcal{U}(\mathcal{A}),\theta \sim \Theta}[Q_{\theta}(s,a) - Q_{\theta}(s,a^{\textrm{min}})]$
where $a^{\textrm{min}} = \argmin_a Q_{\theta}(s,a)$.
\end{definition}
The following theorem captures the intuition that choosing counteractive actions, i.e. the action minimizing the state-action value function, will achieve an above-average temporal difference. 
\begin{theorem}[\emph{Counteractive Actions Increase Temporal Difference}]
\label{tdprop}
Let $\eta, \delta > 0$ and suppose that $Q_\theta(s,a)$ is $\eta$-uninformed and $\delta$-smooth.
Let $s_t \in \mathcal{S}$ be a state, and let $a^{\textrm{min}}$ be the action minimizing the state-action value in a given state $s_t$, $a^{\textrm{min}} = \argmin_a Q_\theta(s_t,a)$.
Let $s_{t+1}^{\textrm{min}}\sim \mathcal{T}(s_t,a^{\textrm{min}},\cdot)$.
Then for an action $a_t\sim \mathcal{U}(\mathcal{A})$ with $s_{t+1} \sim \mathcal{T}(s_t,a_t,\cdot)$ we have
\begin{align*}
&\mathbb{E}_{s_{t+1}^{\textrm{min}}\sim \mathcal{T}(s_t,a^{\textrm{min}},\cdot), \theta \sim \Theta}[ r(s_t,a^{\textrm{min}})  
+ \gamma  \max_a Q_\theta(s_{t+1}^{\textrm{min}},a) 
-   Q_\theta(s_t,a^{\textrm{min}})]  \\
&>\mathbb{E}_{a_t\sim\mathcal{U}, (\mathcal{A})s_{t+1} \sim \mathcal{T}(s_t,a_t,\cdot), \theta \sim \Theta}[ r(s_t,a_t) 
+ \gamma \max_a Q_\theta(s_{t+1},a) -  Q_\theta(s_t,a_t)] 
+ \mathcal{D}(s_t) - 2\delta - \eta
\end{align*}
\end{theorem}
\begin{proof}
Since $Q_{\theta}(s,a)$ is $\delta$-smooth we have
\begin{align*}
&\mathbb{E}_{s_{t+1}^{\emph{min}}\sim \mathcal{T}(s_t,a^{\emph{min}},\cdot), \theta \sim \Theta}
[ \gamma  \max_a Q_\theta(s_{t+1}^{\emph{min}},a) -  Q_\theta(s_t,a^{\emph{min}})] \\
& \qquad\qquad > \gamma \mathbb{E}_{\theta \sim \Theta}[\max_a Q_\theta(s_{t},a)] - \delta - \mathbb{E}_{\theta \sim \Theta}[Q_\theta(s_t,a^{\emph{min}})]\\
& \qquad\qquad > \gamma  \mathbb{E}_{s_{t+1} \sim \mathcal{T}(s_t,a_t,\cdot), \theta \sim \Theta}
[\max_a Q_\theta(s_{t+1},a)] - 2\delta 
 - \mathbb{E}_{\theta \sim \Theta}[Q_\theta(s_t,a^{\emph{min}})]\\
& \qquad\qquad \geq \mathbb{E}_{a_t \sim \mathcal{U}(\mathcal{A}), s_{t+1} \sim \mathcal{T}(s_t,a_t,\cdot), \theta \sim \Theta}[\gamma \max_a Q_\theta(s_{t+1},a) 
 - Q_\theta(s_t,a_{t})] + \mathcal{D}(s_t) - 2\delta
\end{align*}
where the last line follows from Definition \ref{disadvantagedef}.
Further, because $Q_{\theta}(s,a)$ is $\eta$-uninformed, $\mathbb{E}_{\theta \sim \Theta}[r(s_t,a^{\textrm{min}})] > \mathbb{E}_{a_t\sim \mathcal{U}(\mathcal{A})}[r(s_t,a_t)] - \eta.$
Combining with the previous inequality completes the proof.
\end{proof}
Theorem \ref{tdprop} shows that counteractive actions, i.e. actions that minimize the state-action value function, in fact increase temporal difference.
Now we will prove that counteractive actions achieve an increase in temporal difference further in the case where action selection and evaluation in the temporal difference are computed with two different sets of weights $\theta$ and $\hat{\theta}$ as in double $Q$-learning.
\begin{definition}[\emph{$\delta$-smoothness for Double-$Q$}]
Let $\delta > 0$. A pair of $Q$-functions parameterized by weights $\theta \sim \Theta$ and $\hat{\theta} \sim \Theta$ are $\delta$-smooth if for any state $s \in \mathcal{S}$ and action $\hat{a}=\hat{a}(s,\theta) \in \mathcal{A}$ with $s' \sim \mathcal{T}(s,\hat{a},\cdot)$, we have 
\begin{align*}
\Bigg\vert 
\underset{\substack{\theta, \hat{\theta} \sim \Theta \\ s' \sim \mathcal{T}(s,\hat{a},\cdot)}}{\mathbb{E}}
\left[ Q_{\hat{\theta}}(s, \argmax_a Q_{\theta}(s,a))\right] 
-\underset{\substack{\theta, \hat{\theta} \sim \Theta \\ s' \sim \mathcal{T}(s,\hat{a},\cdot)}}{\mathbb{E}}
 \left[Q_{\hat{\theta}}(s',\argmax_a Q_{\theta}(s',a))\right]\Bigg\vert < \delta
\end{align*}
where the expectation is over both the random initialization of the $Q$-function weights $\theta$ and $\hat{\theta}$, and the random transition to state $s' \sim \mathcal{T}(s,\hat{a},\cdot)$.
\end{definition}
Now we will prove that counteractive actions, i.e. actions that minimize the state-action value instead of maximizing, will lead to increase in temporal difference in the case of two $Q$-functions, i.e. $Q_{\theta}$ and $Q_{\hat{\theta}}$, that are alternatively used to take an action and evaluate the value of the action.
\begin{theorem}
\label{doubleqprop}
Let $\eta,\delta > 0$ and suppose that $Q_{\theta}$ and $Q_{\hat{\theta}}$ are $\eta$-uniformed and $\delta$-smooth.
Let $s_t \in \mathcal{S}$ be a state, and let $a^{\textrm{min}} = \argmin_a Q_{\theta}(s_t,a)$.
Let $s_{t+1}^{\textrm{min}}\sim \mathcal{T}(s_t,a^{\textrm{min}},\cdot)$.
Then for an action $a_t\sim \mathcal{U}(\mathcal{A})$ with $s_{t+1} \sim \mathcal{T}(s_t,a_t,\cdot)$ we have
\begin{align*}
&\mathbb{E}_{s_{t+1} \sim \mathcal{T}(s,a,\cdot), \theta \sim \Theta, \hat{\theta} \sim \Theta}[ r(s_t,a^{\textrm{min}}) 
+ \gamma Q_{\hat{\theta}}(s_{t+1}^{\textrm{min}},\argmax_a Q_{\theta}(s_{t+1}^{\textrm{min}},a)) -  Q_{\theta}(s_t,a^{\textrm{min}})] \\
& \quad> \mathbb{E}_{a_t \sim \mathcal{U}(\mathcal{A}), s_{t+1} \sim \mathcal{T}(s,a,\cdot), \theta \sim \Theta, \hat{\theta} \sim \Theta}[ r(s_t,a_t)
+ \gamma Q_{\hat{\theta}}(s_{t+1},\argmax_a Q_{\theta}(s_{t+1},a)) -  Q_{\theta}(s_t,a_t)] \\
&\quad\quad\quad +\mathcal{D}(s_t) - 2\delta -\eta
\end{align*}
\end{theorem}
\begin{proof}
Since $Q_{\theta}$ and $Q_{\hat{\theta}}$ are $\delta$-smooth we have
\begin{align*}
&\mathbb{E}_{s_{t+1}^{\emph{min}} \sim \mathcal{T}(s_t,a^{\emph{min}},\cdot), \theta \sim \Theta, \hat{\theta} \sim \Theta}[ 
+ \gamma Q_{\hat{\theta}}(s_{t+1}^{\emph{min}},
 \argmax_a Q_{\theta}(s_{t+1}^{\emph{min}},a)) -  Q_{\theta}(s_t,a^{\emph{min}})]  \\
& \quad >\mathbb{E}_{s_{t+1}^{\emph{min}} \sim \mathcal{T}(s_t,a^{\emph{min}},\cdot), \theta \sim \Theta, \hat{\theta} \sim \Theta}[ 
+ \gamma Q_{\hat{\theta}}(s_{t}, 
\argmax_a Q_{\theta}(s_{t},a)) -  Q_{\theta}(s_t,a^{\emph{min}})] - \delta \\
& \quad > \mathbb{E}_{s_{t+1} \sim \mathcal{T}(s_t,a_t,\cdot), \theta \sim \Theta, \hat{\theta} \sim \Theta}[   
+ \gamma Q_{\hat{\theta}}(s_{t+1}, 
 \argmax_a Q_{\theta}(s_{t+1},a)) -  Q_{\theta}(s_t,a^{\emph{min}})] - 2\delta\\
& \quad \geq \mathbb{E}_{s_{t+1} \sim \mathcal{T}(s_t,a_t,\cdot), \theta \sim \Theta, \hat{\theta} \sim \Theta}[  
+ \gamma Q_{\hat{\theta}}(s_{t+1},
 \argmax_a Q_{\theta}(s_{t+1},a)) -  Q_{\theta}(s_t,a_t)] + \mathcal{D}(s_t) -2\delta
\end{align*}
where the last line follows from Definition \ref{disadvantagedef}.
Further, because $Q_{\theta}$ and $Q_{\hat{\theta}}$  are $\eta$-uniformed, $\mathbb{E}_{\theta \sim \Theta, \hat{\theta}\sim\Theta}[r(s_t,a^{\textrm{min}})] > \mathbb{E}_{a_t\sim \mathcal{U}(\mathcal{A})}[r(s_t,a_t)] - \eta$. 
Combining with the previous inequality completes the proof.
\end{proof}
\vskip-0.15in
\textbf{Core Counterintuition:}  
\hskip1in
\begin{GrayBox}
\emph{\hskip0.4in How could minimizing the state-action value function accelerate learning?}
\end{GrayBox}
\vskip 0.04in
At first, the results in Theorem \ref{tdprop} and \ref{doubleqprop} might appear counterintuitive. Yet, understanding this counterintuitive fact relies on first understanding the intrinsic difference between the randomly initialized
state-action value function, i.e. $Q_{\theta}$, and the optimal state-action value function, i.e. $Q^*$. In particular, from the perspective of the function $Q^*$, the action $a_{Q_{\theta}}^{\text{min}}(s) = \argmin_a Q_{\theta}(s,a)$ is a uniform random action. However, from the perspective of the function $Q_{\theta}$, the action $a^{\text{min}}$ is meaningful, in that it will lead to a higher TD-error update than any other action; hence the realization of the intrinsic difference between $a^{\text{min}}_{Q_{\theta}}(s)$ and $a^{\text{min}}_{Q^*}(s)$ with regard to $Q_{\theta}$ and $Q^*$ provides a valuable insight on how counteractive actions do in fact increase temporal difference.
In fact, Theorem \ref{tdprop} and \ref{doubleqprop} precisely provide the formalization that the temporal difference achieved by taking the minimum action is larger than that of a random action by an amount equal to the disadvantage gap $\mathcal{D}(s)$.
Experimental results reported in Section \ref{largeexp} further verify the theoretical analysis. 
Now we will formalize this intuition for initialization and prove that the distribution of the minimum value action in a given state is uniform by itself, but is constant once it is conditioned on the weights $\theta$.
\begin{proposition}[\emph{Marginal and Conditional Distribution of Counteractive Actions}]
    \label{initialweight}
Let $\theta$ be the random initial weights for the $Q$-function. For any state $s\in \mathcal{S}$ let $a^{\textrm{min}}(s) = \argmin_{a' \in \mathcal{A}}Q_{\theta}(s,a')$. Then for any $a \in \mathcal{A}$, $\mathbb{P}_{\theta \sim \Theta}\left[\argmin_{a' \in \mathcal{A}}Q_{\theta}(s,a')=a\right] = \frac{1}{\lvert \mathcal{A}\rvert}$
i.e. the distribution $\mathbb{P}_{\theta  \sim \Theta}[a^{\textrm{min}}(s)]$ is uniform.
Simultaneously, the conditional distribution $\mathbb{P}_{\theta \sim \Theta}[a^{\textrm{min}}(s)\mid \theta]$ is constant.
\end{proposition}

\begin{algorithm}[t]
\caption{CoAct TD Learning: Counteractive Temporal Difference Learning}\label{alg:exp}
\begin{algorithmic}
\STATE{\textbf{Input:} In MDP $\mathcal{M}$ with $\gamma \in (0,1]$, $s \in \mathcal{S}$, $a \in \mathcal{A}$ with $Q_\theta(s,a)$ function parametrized by $\theta$, $\epsilon$ dithering parameter, $\mathcal{B}$ experience replay buffer, $\mathcal{N}$ is the training learning steps.}
\end{algorithmic}
\vspace{-0.5cm} 
\begin{multicols}{2}
\vspace{-1cm} 
\begin{algorithmic}
\vskip-0.4in
\vskip-0.4in
\hsize=0.4\textwidth\linewidth=\hsize\columnwidth=\hsize
\STATE{  Populating Experience Replay Buffer:}
\FOR{$s_t$ in $e$ }
\STATE{Sample $\kappa \sim U(0,1)$}
\IF{$\kappa < \epsilon$}
\STATE{$a^{\textit{min}} = \argmin_a Q(s_t,a)$}
\STATE{$s^{\textit{min}}_{t+1} \sim \mathcal{T}(s_t,a^{\textit{min}},\cdot)$}
\STATE{$\mathcal{B} \leftarrow (r(s_t, a^{\textit{min}}), s_t, s^{\textit{min}}_{t+1}, a^{\textit{min}})$}
\ELSE
\STATE{$a^{\textrm{max}} = \argmax_a Q(s_t,a)$}
\STATE{$s_{t+1} \sim \mathcal{T}(s_t,a^{\textrm{max}},\cdot)$}
\STATE{$\mathcal{B} \leftarrow (r(s_t, a^\textrm{max}), s_t, s_{t+1}, a^{\textrm{max}})$}
\ENDIF
\ENDFOR
\end{algorithmic}
\columnbreak
\vspace{-1cm} 
\begin{algorithmic}
\STATE{Learning:}
\FOR{ $n$ in $\mathcal{N}$ }
\STATE{\hspace{-0.09cm}Sample from replay buffer:} 
\STATE{\hspace{-0.09cm}$\langle s_t, a_t, r(s_t,a_t), s_{t+1}  \rangle \sim \mathcal{B} $} 
\STATE{\hspace{-0.09cm}Thus, $\mathcal{TD}$ receives update with probability $\epsilon$}
\STATE{\vspace{-0.3cm}\hspace{-0.09cm}$\mathcal{TD} = r(s_t,a^{\textit{min}})+ \gamma \max_a Q(s^{\textit{min}}_{t+1}, a)- Q(s_t,a^{\textit{min}}) $}
\STATE{\hspace{-0.09cm}$\mathcal{TD}$ receives update with probability $1-\epsilon$:}
\STATE{ $\mathcal{TD} = r(s_t,a^{\textrm{max}})+ \gamma \max_a Q(s_{t+1}, a)- Q(s_t,a^{\textrm{max}}) $}
\STATE Gradient step with $\nabla \mathcal{L}(\mathcal{TD})$
\hspace{-1cm} \ENDFOR
\end{algorithmic}
\end{multicols}
\vspace{-0.4cm} 
\end{algorithm}

\vskip-0.1in

The proof is provided in the supplementary material. 
Proposition \ref{initialweight} shows that in states in which the $Q$-function has not received sufficient updates, taking the minimum action is almost equivalent to taking a random action with respect to its contribution to the rewards obtained.
However, while the action chosen early on in training is almost uniformly random when only considering the current state, 
it is at the same time still completely determined by the current value of the weights $\theta$, 
as is the temporal difference.
Thus while the marginal distribution on actions taken is uniform, the temporal difference when taking counteractive actions, i.e. the minimum action, is quite different than from the case where an independently random action is chosen.
In particular, in expectation over the random initialization $\theta \sim \Theta$, the temporal difference is higher when taking the minimum value action than that of a random action as demonstrated in Section \ref{worstq}.

The main objective of our approach is to increase the information gained from each environment interaction,
and we show that this can be achieved via actions that minimize the state-action value function.
While minimization of the $Q$-function may initially be regarded as counterintuitive, Section \ref{worstq} provides the exact theoretical justification on how taking actions that minimize the state-action value function results in higher temporal difference for the corresponding state transitions, and Section \ref{largeexp} provides experimental results that verify the theoretical analysis.
Our method is a fundamental theoretically well-motivated improvement on temporal difference learning. Thus, any algorithm in reinforcement learning that is built upon temporal difference learning can be simply switched to CoAct TD learning.
Algorithm \ref{alg:exp} summarizes our proposed algorithm CoAct TD Learning based on minimizing the state-action value function as described in detail in Section \ref{worstq}.
Note that populating the experience replay buffer and learning are happening simultaneously with different rates. TD receives an update with probability $\epsilon$ solely due to the experience collection. 
\vskip -0.2in
\begin{GrayBox}
\emph{\hskip0.08in CoAct TD is modular: It is a plug-and-play method with any canonical baseline algorithm.}
\end{GrayBox}
\vskip -0.3in
\section{Motivating Example}
\label{motivatingexample}
To truly understand the intuition behind our counterintuitive foundational method, we consider a motivating example: the chain MDP.
In particular, the chain MDP which consists of a chain of $n$ states $s \in \mathcal{S}=\{1,2,\cdots n\}$ each with four actions. Each state $i$ has one action that transitions the agent up the chain by one step to state $i+1$, one action that transitions the agent to state $2$, one action that transitions the agent to state $3$, and one action which resets the agent to state $1$ at the beginning of the chain. All transitions have reward zero, except for the last transition returning the agent to the beginning from the $n$-th state.
Thus, when started from the first state in the chain, the agent must learn a policy that takes $n-2$ consecutive steps up the chain, and then one final step to reset and get the reward.
For the chain MDP, we compare standard approaches in temporal difference learning in tabular $Q$-learning with our method CoAct TD Learning based on minimization of the state-action values. In particular we compare our method CoAct TD Learning with both the $\epsilon$-greedy action selection method, and the upper confidence bound (UCB) method. 
In more detail, in the UCB method the number of training steps $t$, and the number of times $N_t(s,a)$ that each action $a$ has been taken in state $s$ by step $t$ are recorded. Furthermore, the action $a\in \mathcal{A}$ selection is determined as follows:
\vskip-0.18in
\[
a^{\textrm{UCB}} = \argmax_{a \in \mathcal{A}} Q(s,a) + 2\sqrt{\frac{\log t}{N_t(s,a)}}.
\]
In a given state $s$ if $N(s,a) = 0$ for any action $a$, then an action is sampled uniformly at random from the set of actions $a'$ with $N(s,a')=0$.
For the experiments reported in our paper the length of the chain is set to $n=10$.
The $Q$-function is initialized by independently sampling each state-action value from a normal distribution with $\mu=0$ and $\sigma=0.1$.
\begin{figure*}[t]
\footnotesize
\vskip -0.16in
\begin{center}
\stackunder[3pt]{\includegraphics[scale=0.264]{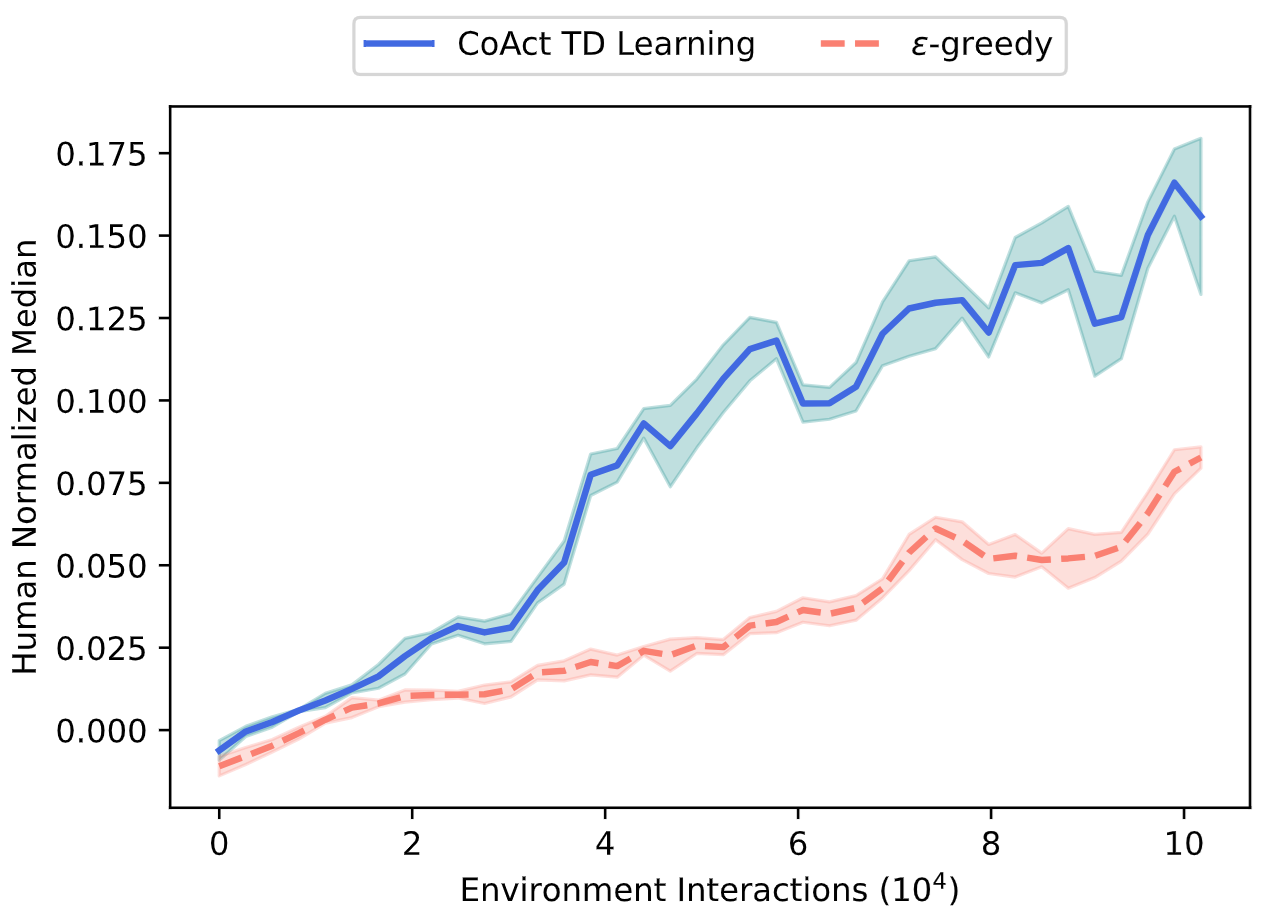}}{}
\stackunder[3pt]{\includegraphics[scale=0.26]{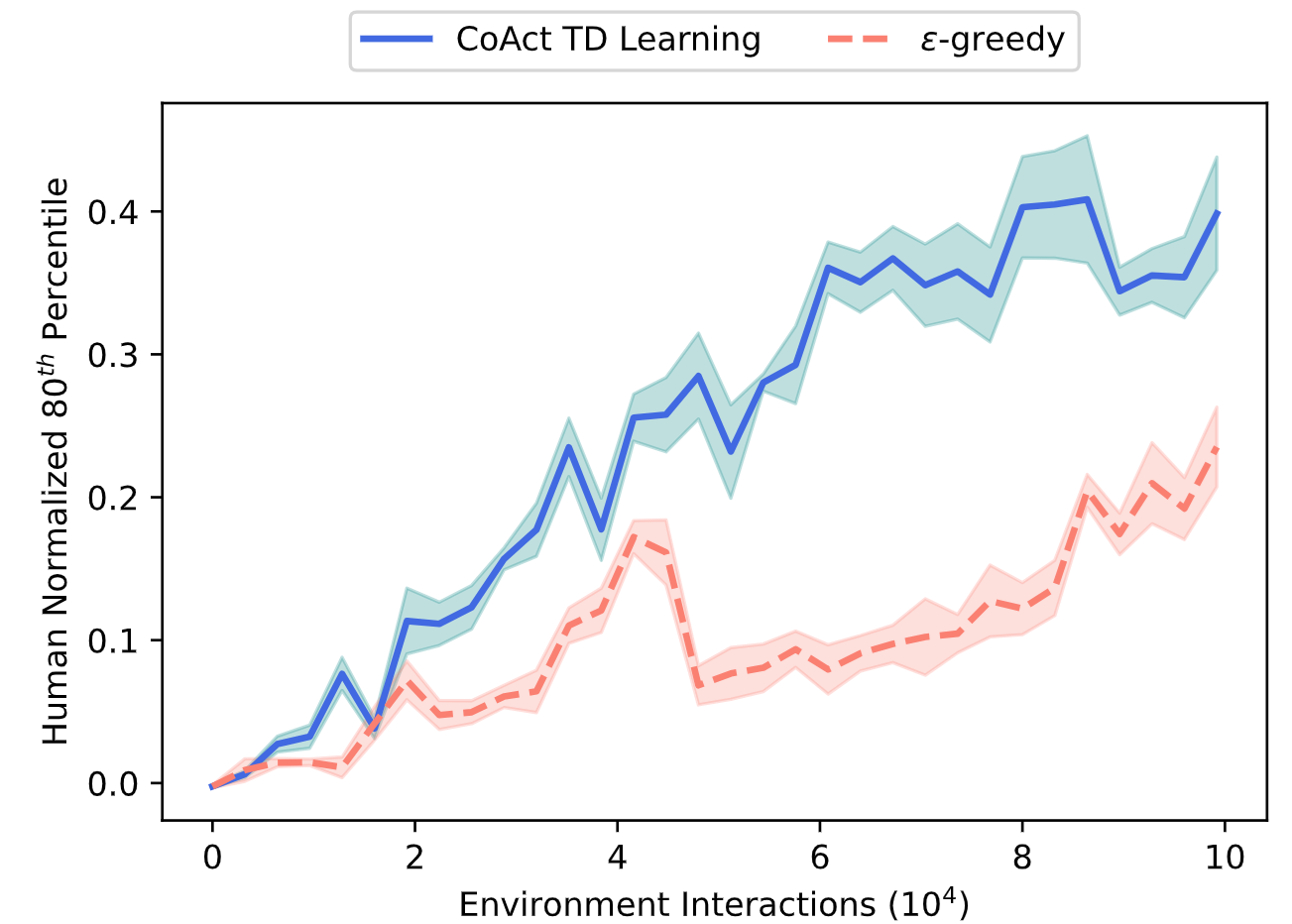}}{}
\end{center}
\vskip -0.22in
\caption{Human normalized scores median and 80$^{\textrm{th}}$ percentile over all games in the Arcade Learning Environment (ALE) 100K benchmark for CoAct TD Learning and the canonical temporal difference learning with $\epsilon$-greedy for QRDQN. Left: Median. Right: 80$^{\textrm{th}}$ Percentile.}
\label{qrdqn}
\vskip -0.18in
\end{figure*}
\begin{figure}[t!]
\centering
\stackunder[1pt]{\includegraphics[scale=0.155]{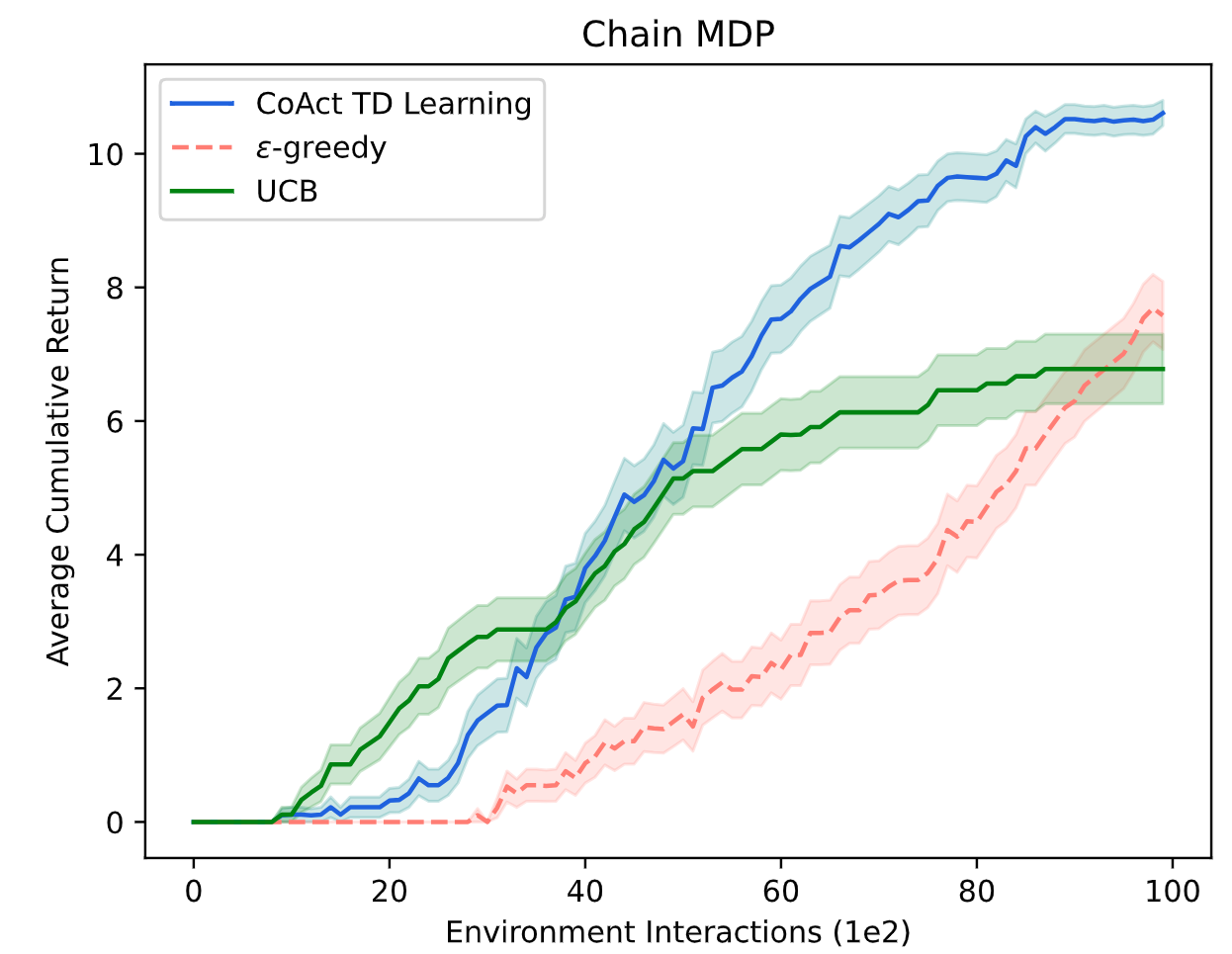}}{}
\stackunder[1pt]{\includegraphics[scale=0.155]{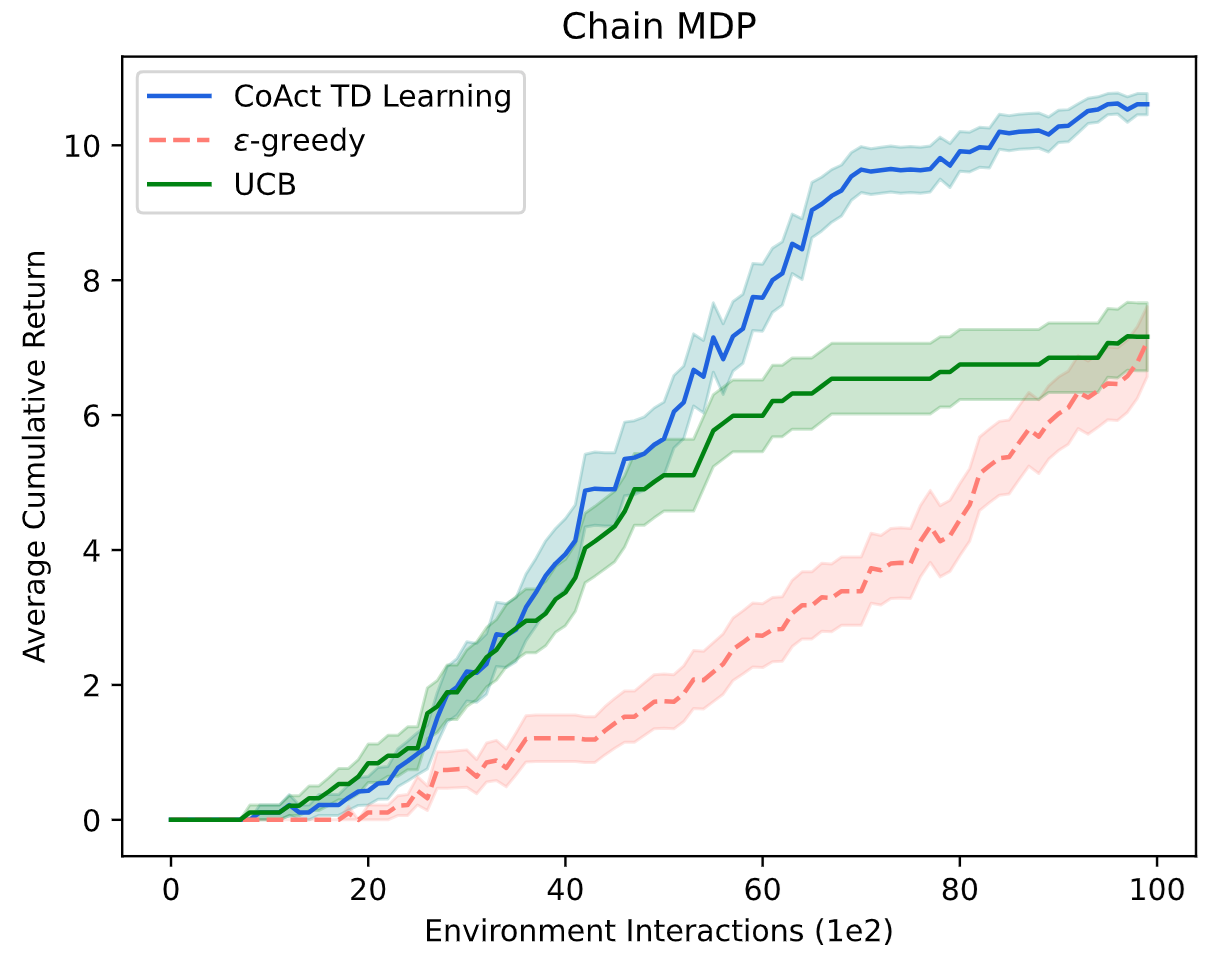}}{} 
\stackunder[1pt]{\includegraphics[scale=0.155]{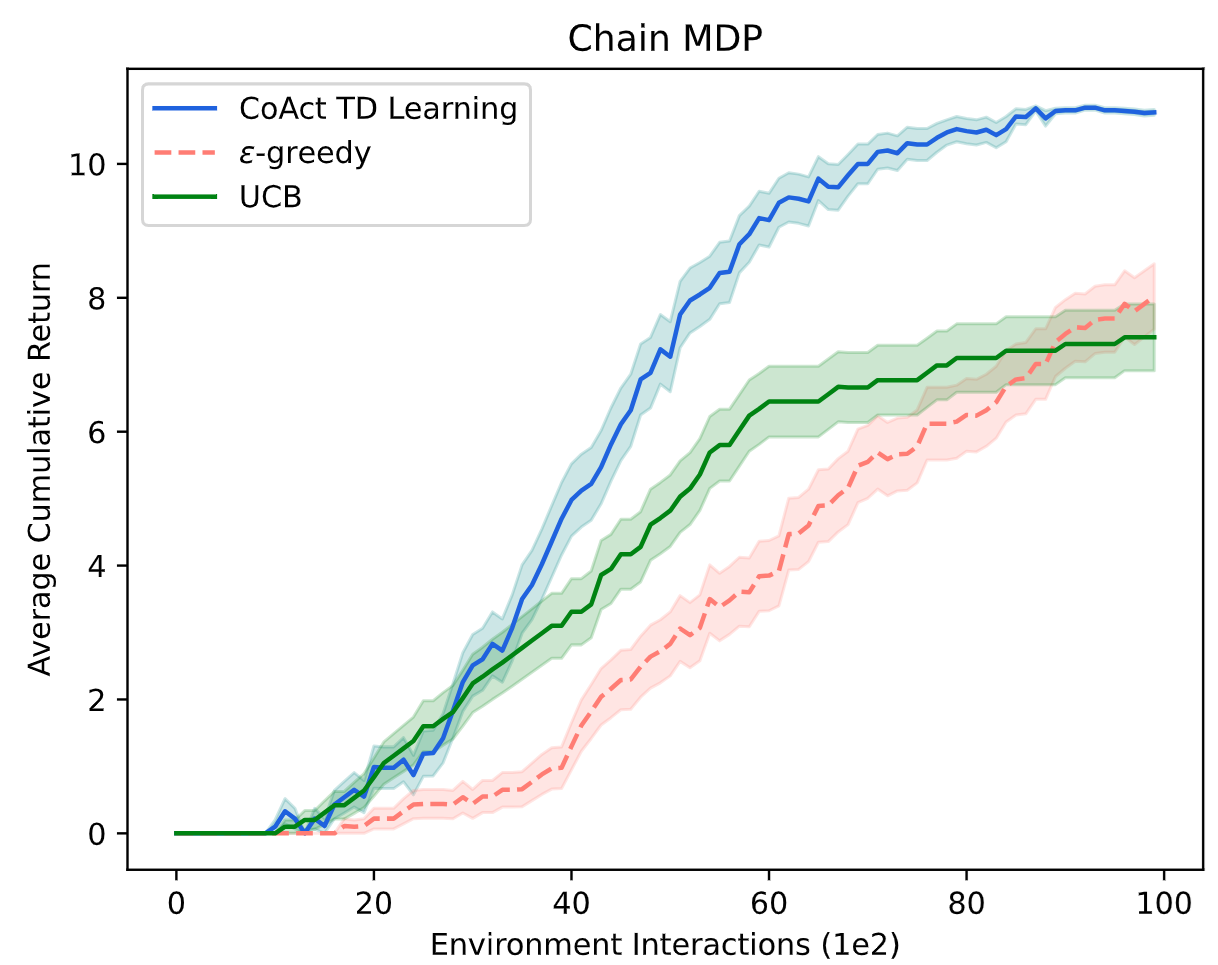}}{}
\stackunder[1pt]{\includegraphics[scale=0.155]{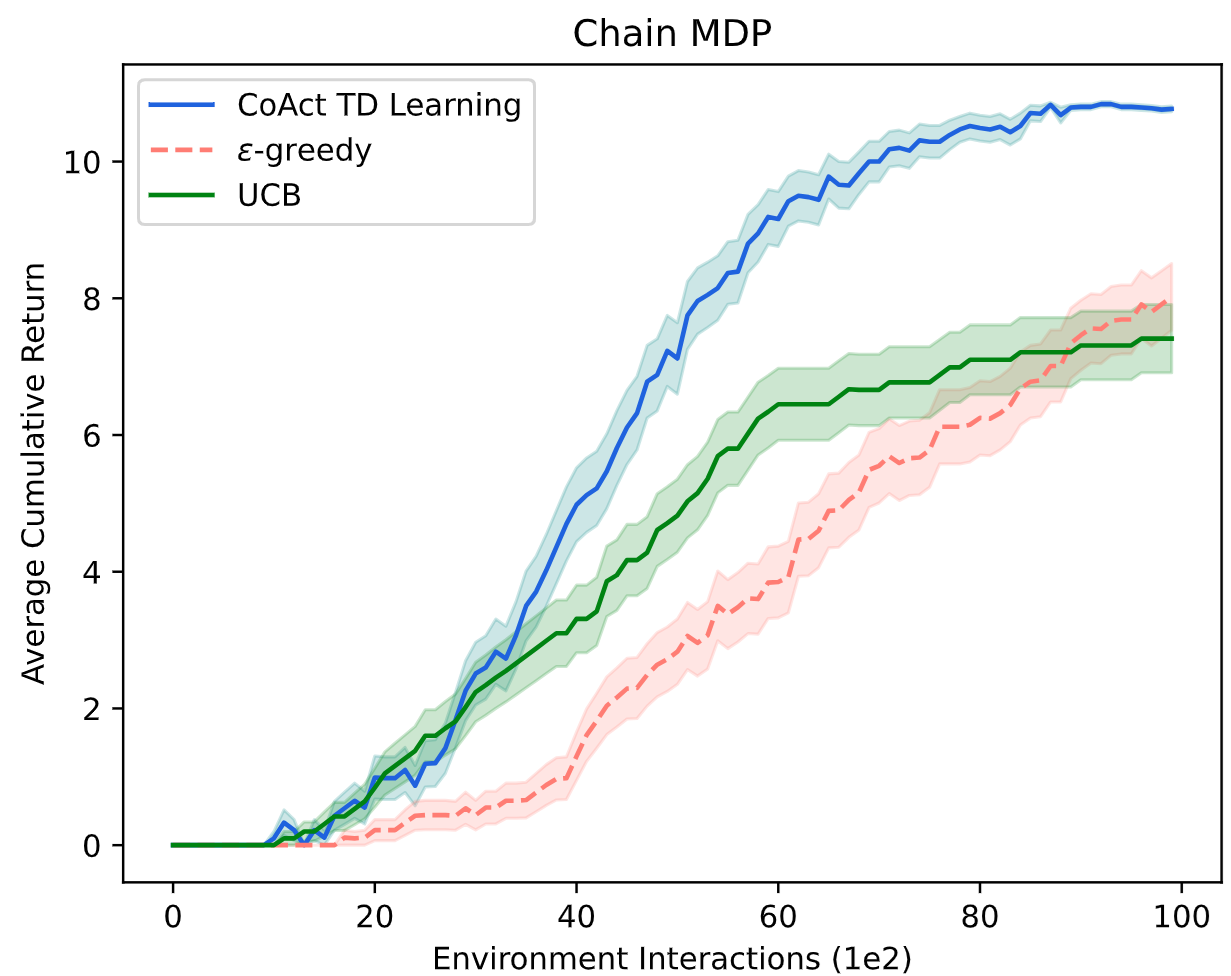}}{}
\vskip -0.1in
\caption{Learning curves in the chain MDP with our proposed algorithm CoAct TD Learning, the canonical algorithm $\epsilon$-greedy and the UCB algorithm with variations in $\epsilon$.}
\label{chainmdp}
\vskip -0.14in
\end{figure}
In each iteration we train the agent using $Q$-learning for 100 steps, and then evaluate the reward obtained by the argmax policy using the current $Q$-function for 100 steps.
Note that the maximum achievable reward in 100 steps is 11.
Figure \ref{chainmdp} reports the learning curves for each method with varying $\epsilon\in[0.15,0.25]$ with step size $0.025$.
The results in Figure \ref{chainmdp} demonstrate that our method converges faster to the optimal policy than either of the standard approaches.

\section{Experimental Results}

\label{largeexp}

The experiments are conducted in the Arcade Learning Environment (ALE). We conduct empirical analysis with multiple baseline algorithms including Deep Double-Q Network \citep{hado16} initially proposed in \citep{hasselt10} trained with prioritized experience replay \citep{tom16} without the dueling architecture with its original version \citep{hado16}, and the QRDQN algorithm that is also described in Section \ref{exploration}.
The experiments are conducted both in the 100K Arcade Learning Environment benchmark, and the canonical version with 200 million frame training \citep{mn15,wang16}. Note that the 100K Arcade Learning Environment benchmark is an established baseline proposed to measure sample efficiency in deep reinforcement learning research, and contains 26 different Arcade Learning Environment games. The policies are evaluated after 100000 environment interactions. All of the polices in the experiments are trained over 5 random seeds. The hyperparameters, the architecture details, and additional experimental results are reported in the supplementary material. All of the results in the paper are reported with the standard error of the mean.
The human normalized scores are computed as, $\textrm{HN} = (\textrm{Score}_{\textit{agent}} - \textrm{Score}_{\textit{random}})/(\textrm{Score}_{\textit{human}} - \textrm{Score}_{\textit{random}})$.
\begin{figure*}[t]
    \footnotesize
    \begin{center}
    \vskip -0.16in
    \stackunder[3pt]{\includegraphics[scale=0.124]{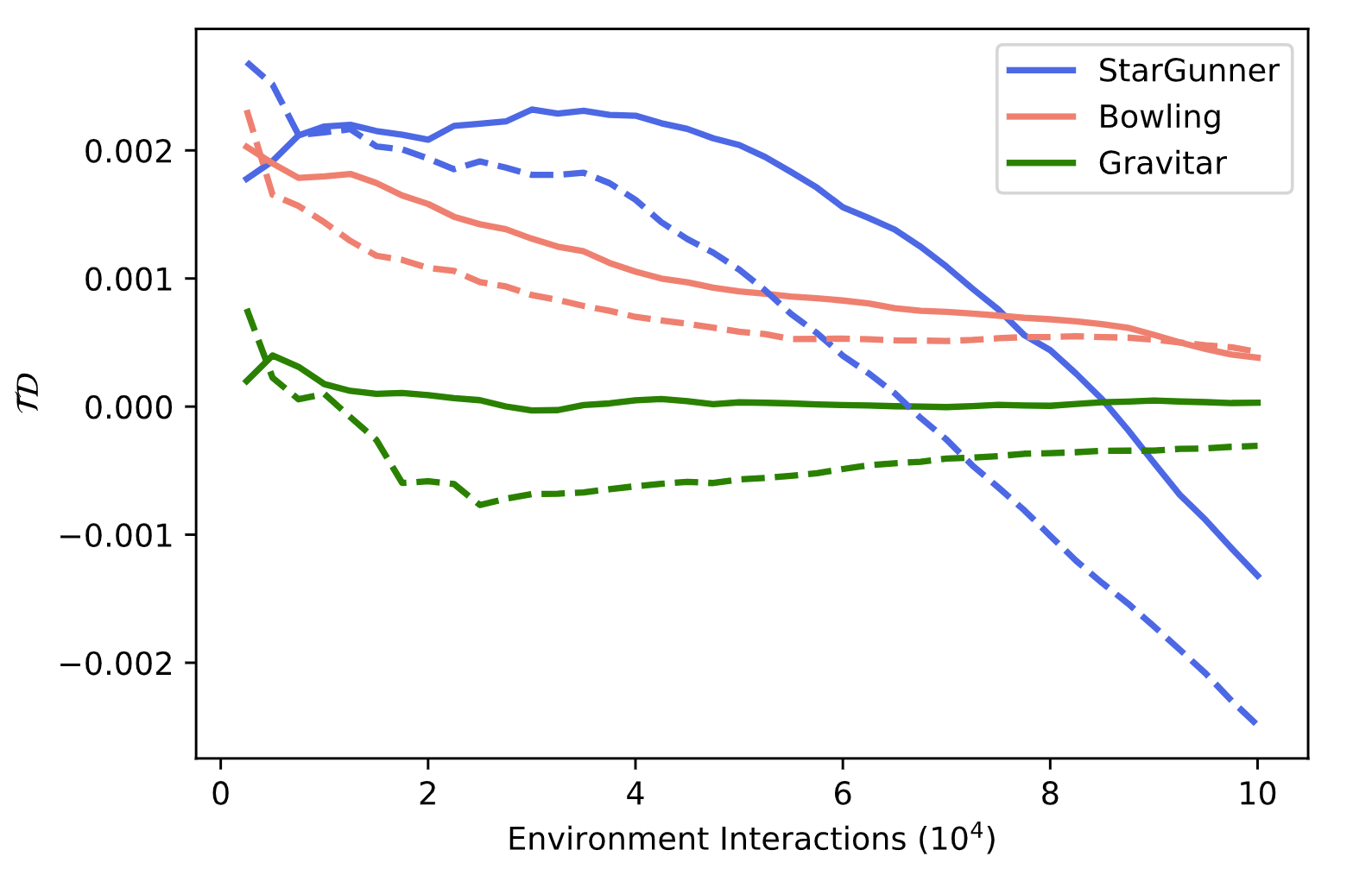}}{\scriptsize{}}
    \hskip-0.05in
    \stackunder[3pt]{\includegraphics[scale=0.124]{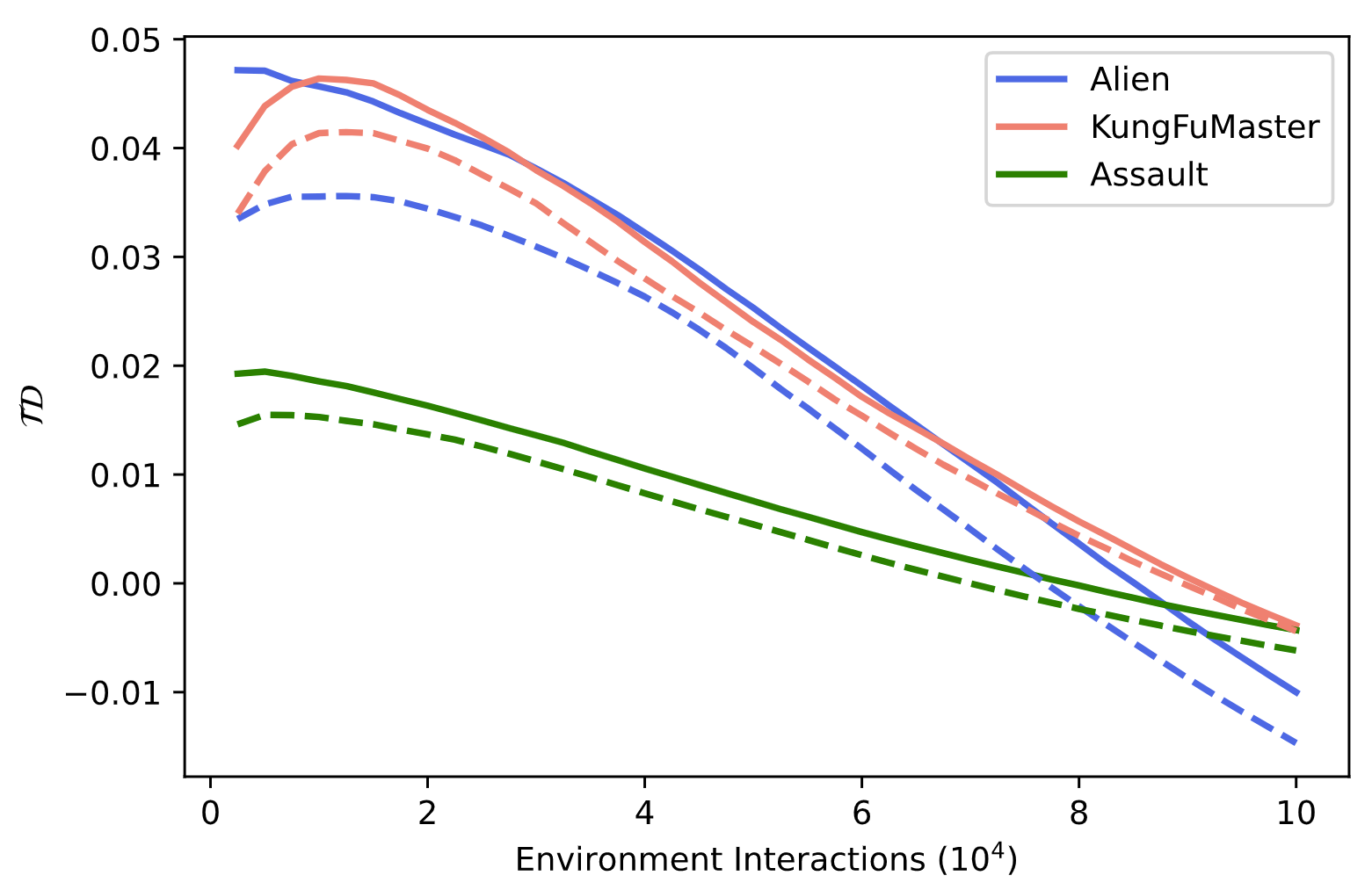}}{\scriptsize{}}
    \stackunder[3pt]{\includegraphics[scale=0.124]{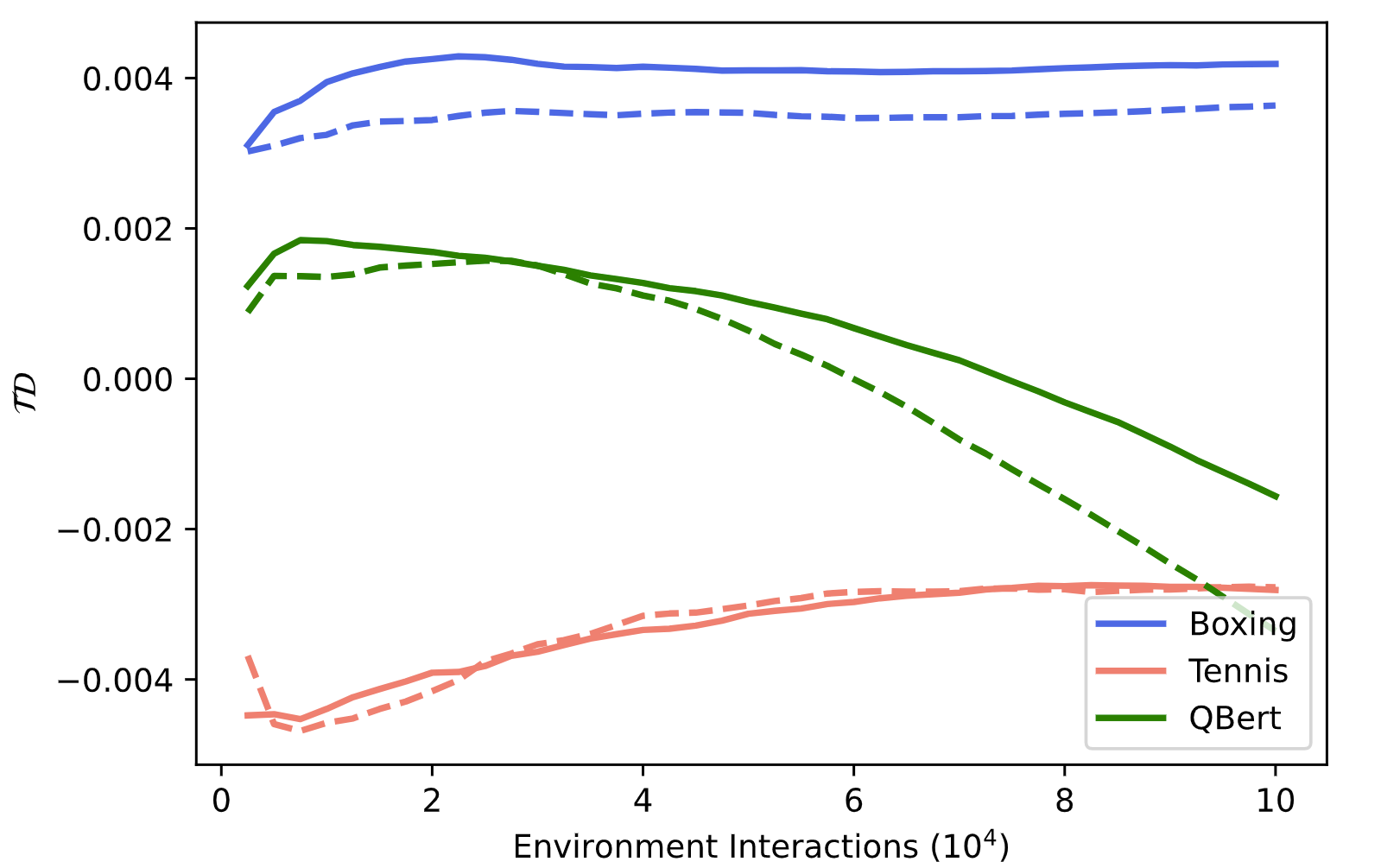}}{\scriptsize{}}
    \stackunder[3pt]{\includegraphics[scale=0.124]{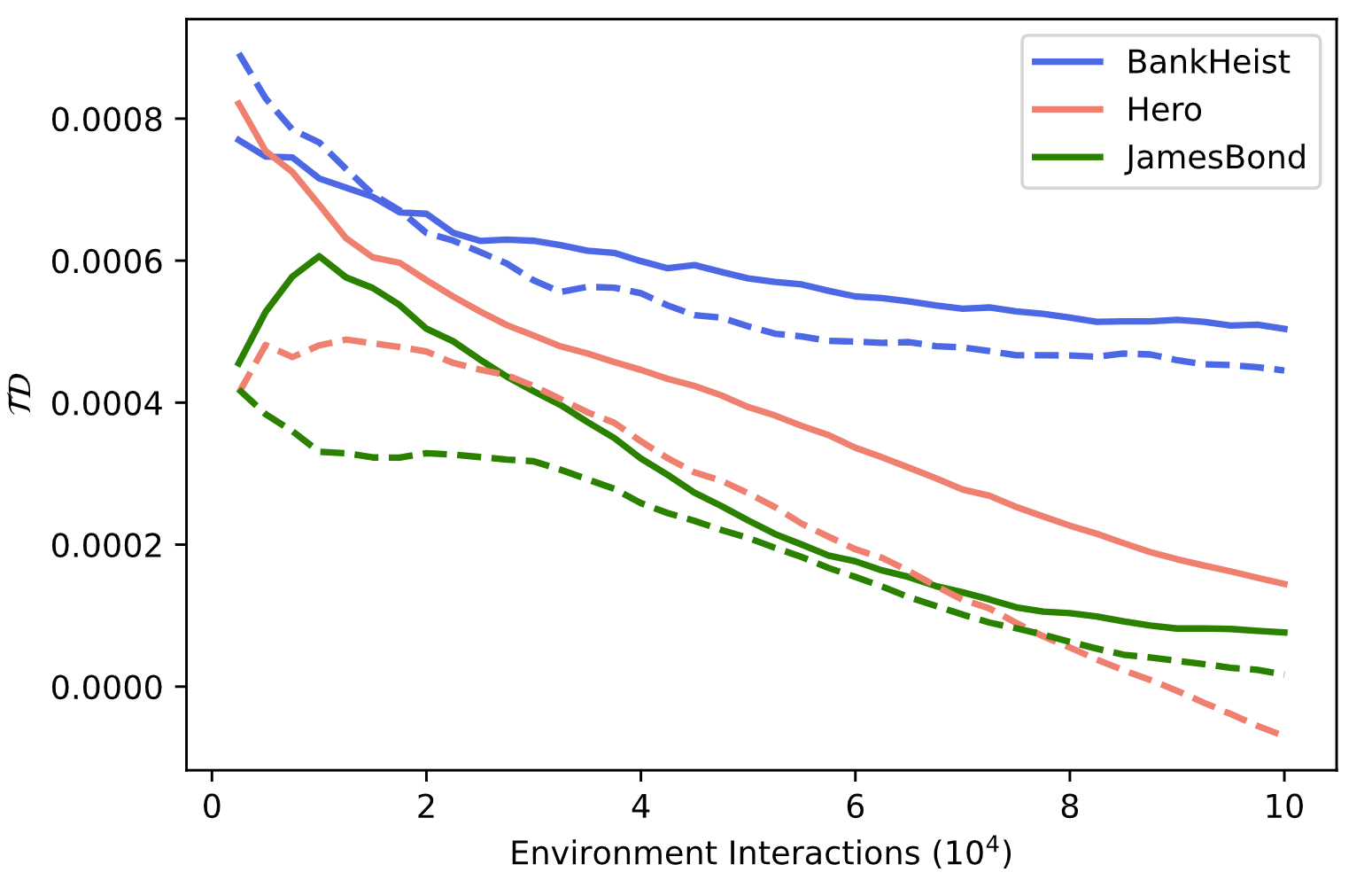}}{\scriptsize{}}
    \end{center}
    \vskip -0.22in
    \caption{Temporal difference for our proposed algorithm CoAct TD Learning and the canonical $\epsilon$-greedy algorithm in the Arcade Learning Environment 100K benchmark. Dashed lines report the temporal difference for the $\epsilon$-greedy algorithm and solid lines report the temporal difference for the CoAct TD Learning algorithm. Colors indicate games.}
    \vskip -0.14in
    \label{TDgames}
    \end{figure*}

\begin{figure*}[t!]
\footnotesize
\begin{center}
\stackunder[3pt]{\includegraphics[scale=0.14]{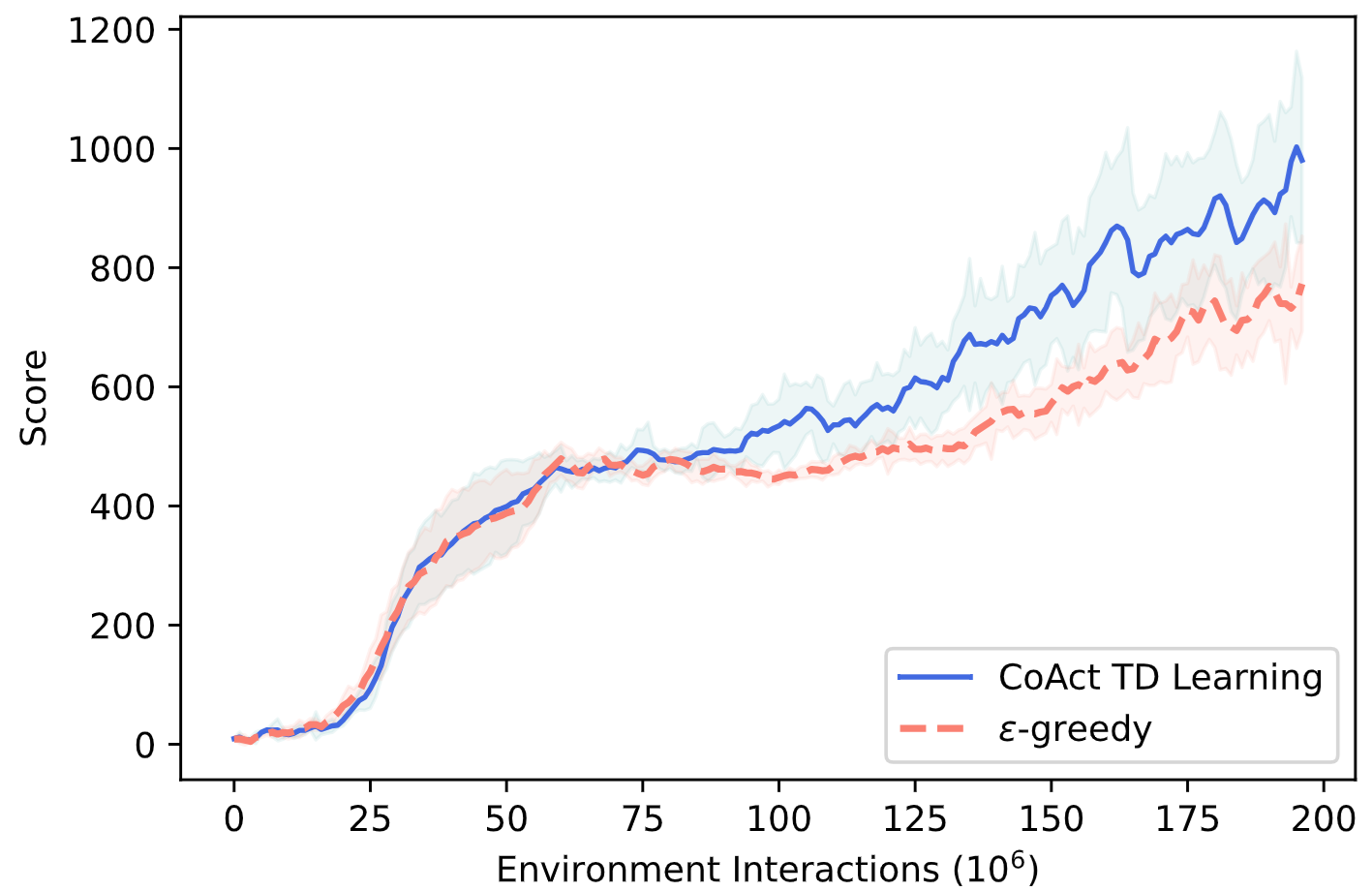}}{\scriptsize{}}
\stackunder[3pt]{\includegraphics[scale=0.14]{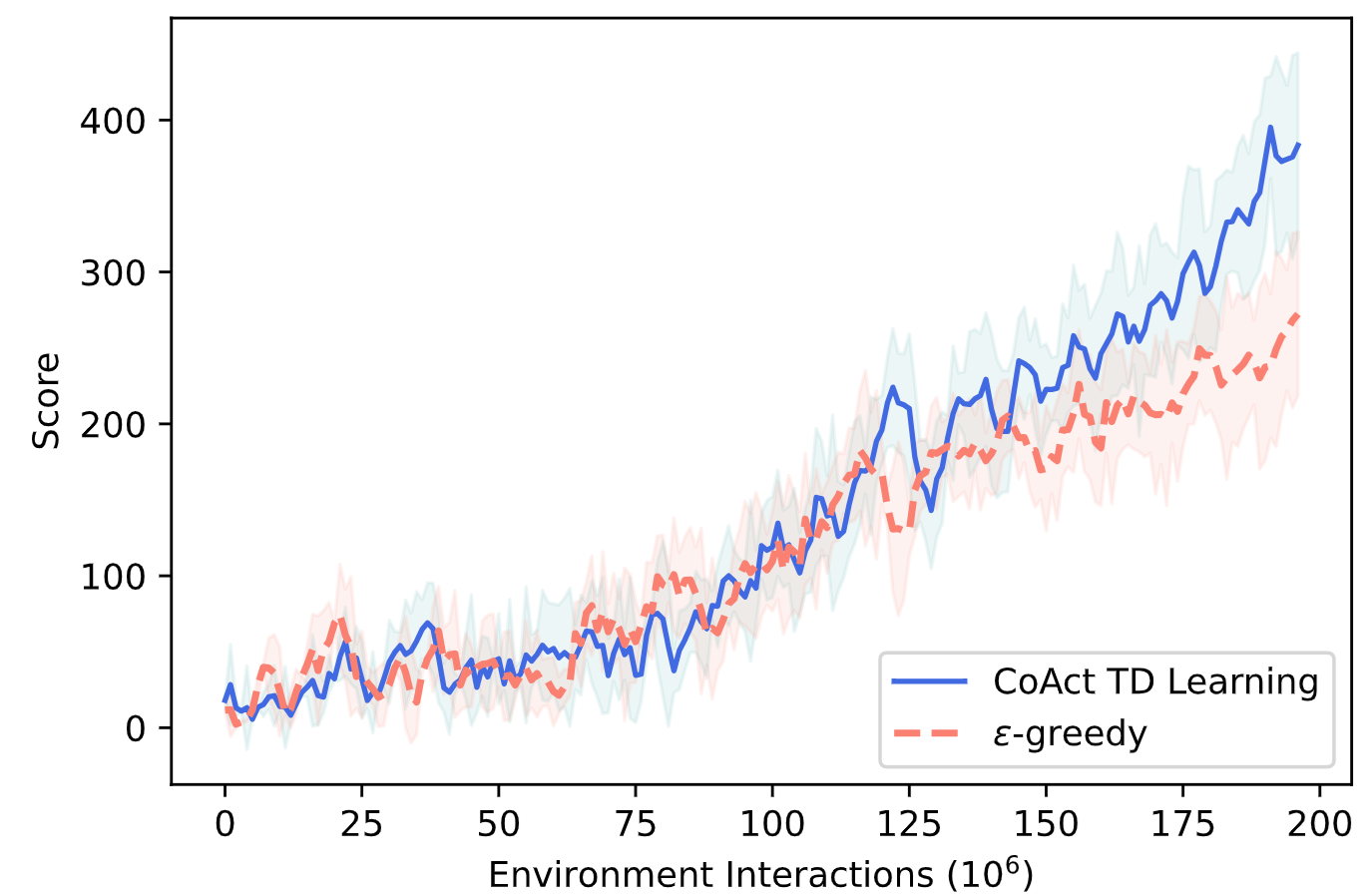}}{\scriptsize{}}
\stackunder[3pt]{\includegraphics[scale=0.14]{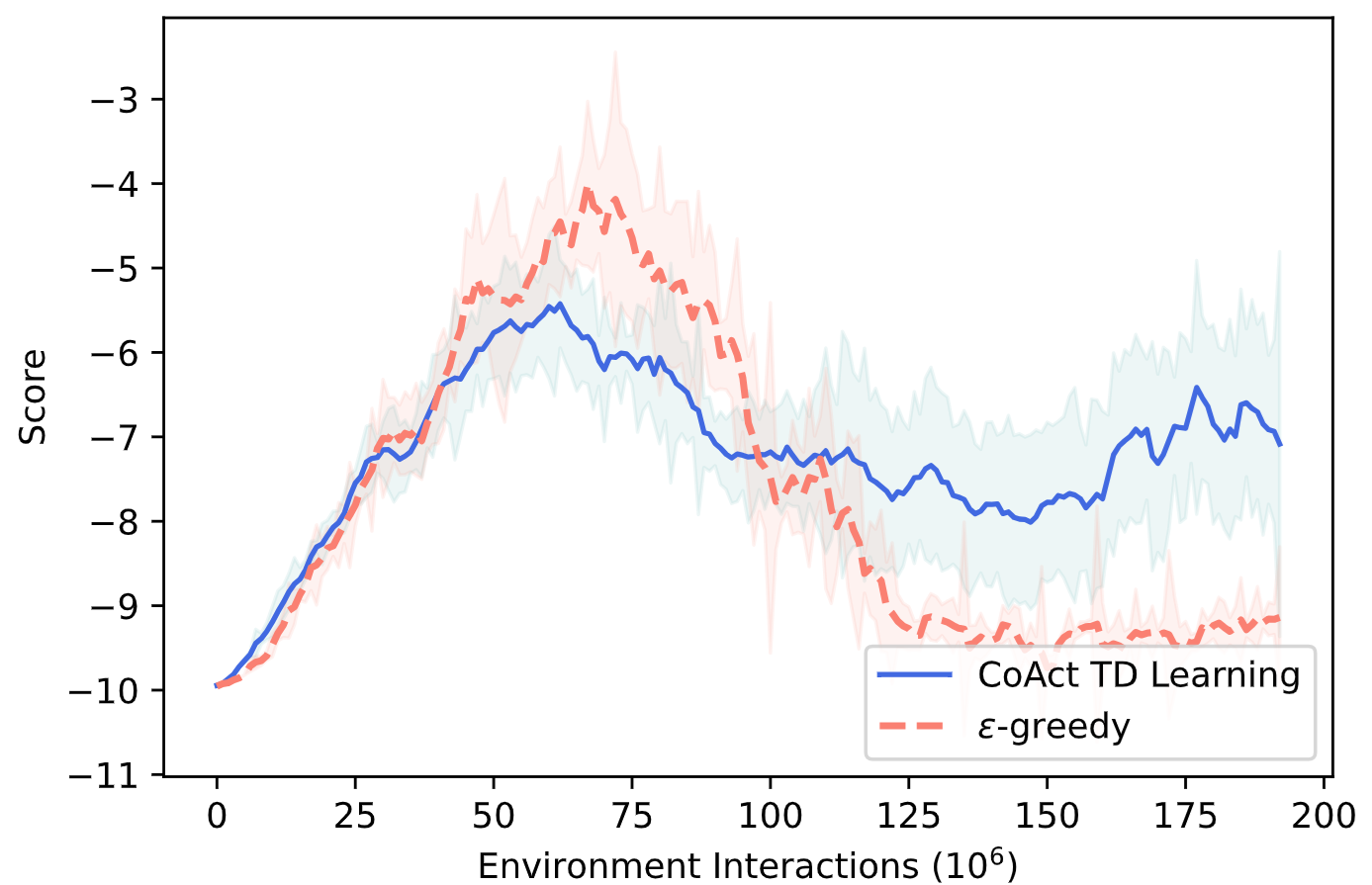}}{\scriptsize{}}
\stackunder[3pt]{\includegraphics[scale=0.14]{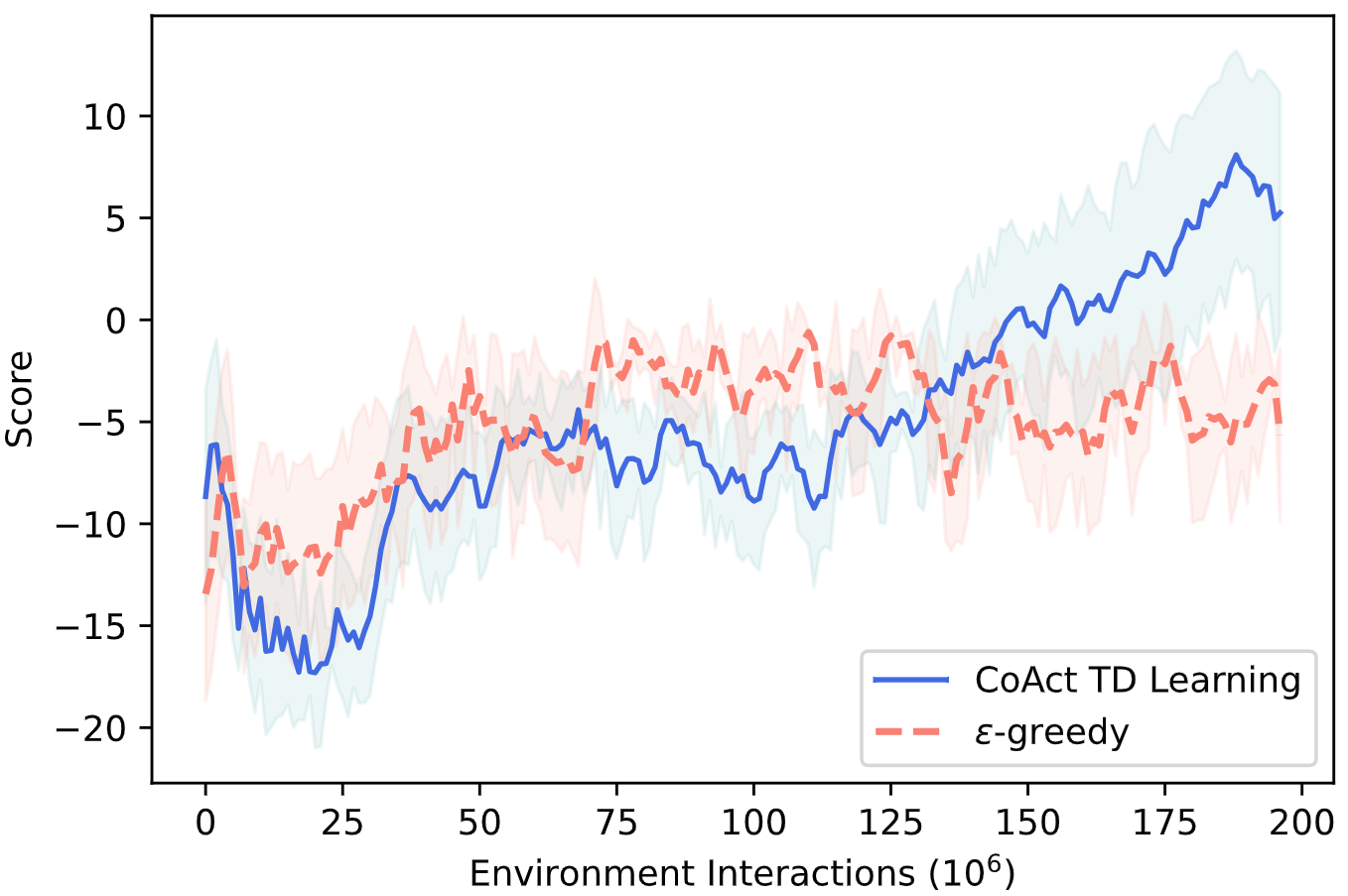}}{\scriptsize{}} \\
\end{center}
\vskip -0.18in
\caption{The learning curves for our proposed method CoAct TD Learning algorithm and canonical temporal difference learning in the Arcade Learning Environment with 200 million frame training. Left: JamesBond. MiddleLeft: Gravitar. MiddleRight: Surround. Right: Tennis. }
\vskip -0.12in
\label{200mil}
\end{figure*}

Figure \ref{qrdqn} reports results of human normalized median scores and $80^{\textrm{th}}$ percentile over all of the games of the Arcade Learning Environment (ALE) in the low-data regime for QRDQN, while Figure \ref{deltatdfigure} reports results for double Q-learning. 
The results reported in Figure \ref{qrdqn} once more demonstrate that the performance obtained by the CoAct TD Learning algorithm is approximately double the performance achieved by the canonical experience collection techniques. 
For completeness we also report several results with 200 million frame training (i.e. 50 million environment interactions).  
Furthermore, we also compare our proposed CoAct TD Learning algorithm with NoisyNetworks as referred to in Section \ref{exploration}.
Table \ref{exptable} reports results of human normalized median scores, 20$^{\textrm{th}}$ percentile, and 80$^{\textrm{th}}$ percentile for the Arcade Learning Environment 100K benchmark. 
Table \ref{exptable} further demonstrates that the CoAct TD Learning algorithm achieves significantly better performance results compared to NoisyNetworks.
Primarily, note that NoisyNetworks includes adding layers in the Q-network to increase exploration. 
However, this increases the number of parameters that have been added in the training process; thus, introducing substantial additional cost.
Figure \ref{200mil} demonstrates the learning curves for our proposed algorithm CoAct TD Learning and the original version of the DDQN algorithm with $\epsilon$-greedy training. In the large data regime we observe that while in some MDPs our proposed method CoAct TD Learning that focuses on experience collection with novel temporal difference boosting via counteractive actions
converges faster, in other MDPs CoAct TD Learning simply converges to a better policy.
\begin{wraptable}{r}{8.5cm}
    \vskip -0.24in
    \caption{Human normalized scores median, 20$^{\textrm{th}}$ and 80$^{\textrm{th}}$ percentile across all of the games in the Arcade Learning Environment 100K benchmark for CoAct TD Learning, $\epsilon$-greedy and NoisyNetworks with DDQN.}
    \label{exptable}
    \centering
    \scalebox{0.8}{
    \begin{tabular}{lccccccr}
    \toprule
    Method                   & CoAct TD 
                & $\epsilon$-greedy
                & NoisyNetworks \\
    \midrule
    Median                             & \textbf{0.0927$\pm$0.0050}     &  0.0377$\pm$0.0031         &   0.0457$\pm$0.0035 \\
    20$^{\textrm{th}}$ Percentile       & \textbf{0.0145$\pm$0.0003}     & 0.0056$\pm$0.0017          &    0.0102$\pm$0.0018  \\
    80$^{\textrm{th}}$ Percentile       &  \textbf{0.3762$\pm$0.0137}    & 0.2942$\pm$0.0233          &   0.1913$\pm$0.0144    \\
    \bottomrule
    \end{tabular}
    }
    \vskip -0.16in
    \end{wraptable} 
Table \ref{exptable} demonstrates that our proposed CoAct TD Learning algorithm improves on the performance of the canonical algorithm $\epsilon$-greedy by 248\% and NoisyNetworks by 204\%.
The results reported in both of the sample regimes demonstrate that CoAct TD learning results in faster convergence rate and significantly improves sample-efficiency in deep reinforcement learning.
The large scale experimental analysis verifies the theoretical predictions of Section \ref{worstq}, and further discovers that the CoAct TD Learning algorithm achieves substantial sample-efficiency with zero-additional cost across many algorithms and different sample-complexity regimes over canonical baseline alternatives.

\begin{figure*}[t]
\footnotesize
\vskip -0.155in
\begin{center}
\stackunder[3pt]{\includegraphics[scale=0.25]{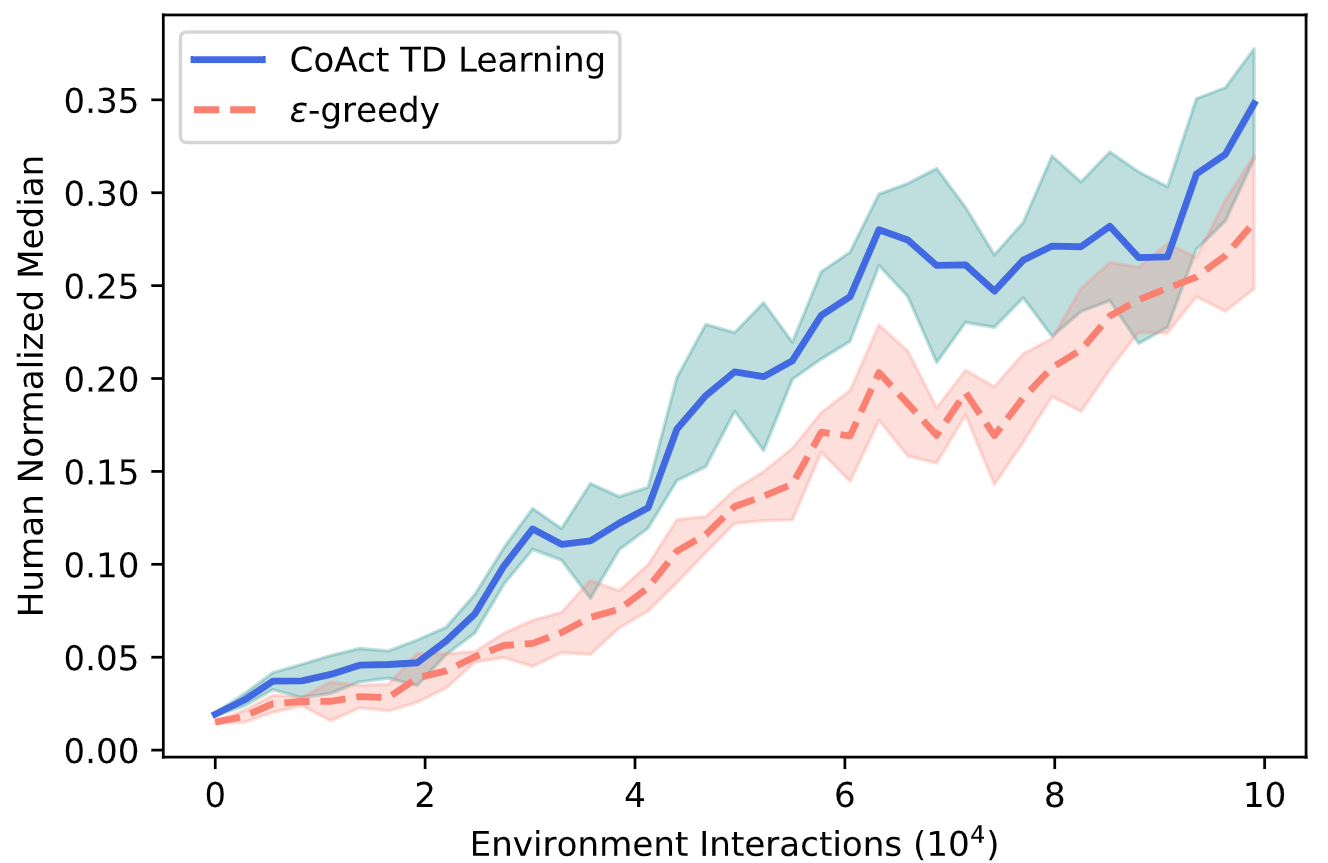}}{}
\stackunder[3pt]{\includegraphics[scale=0.24]{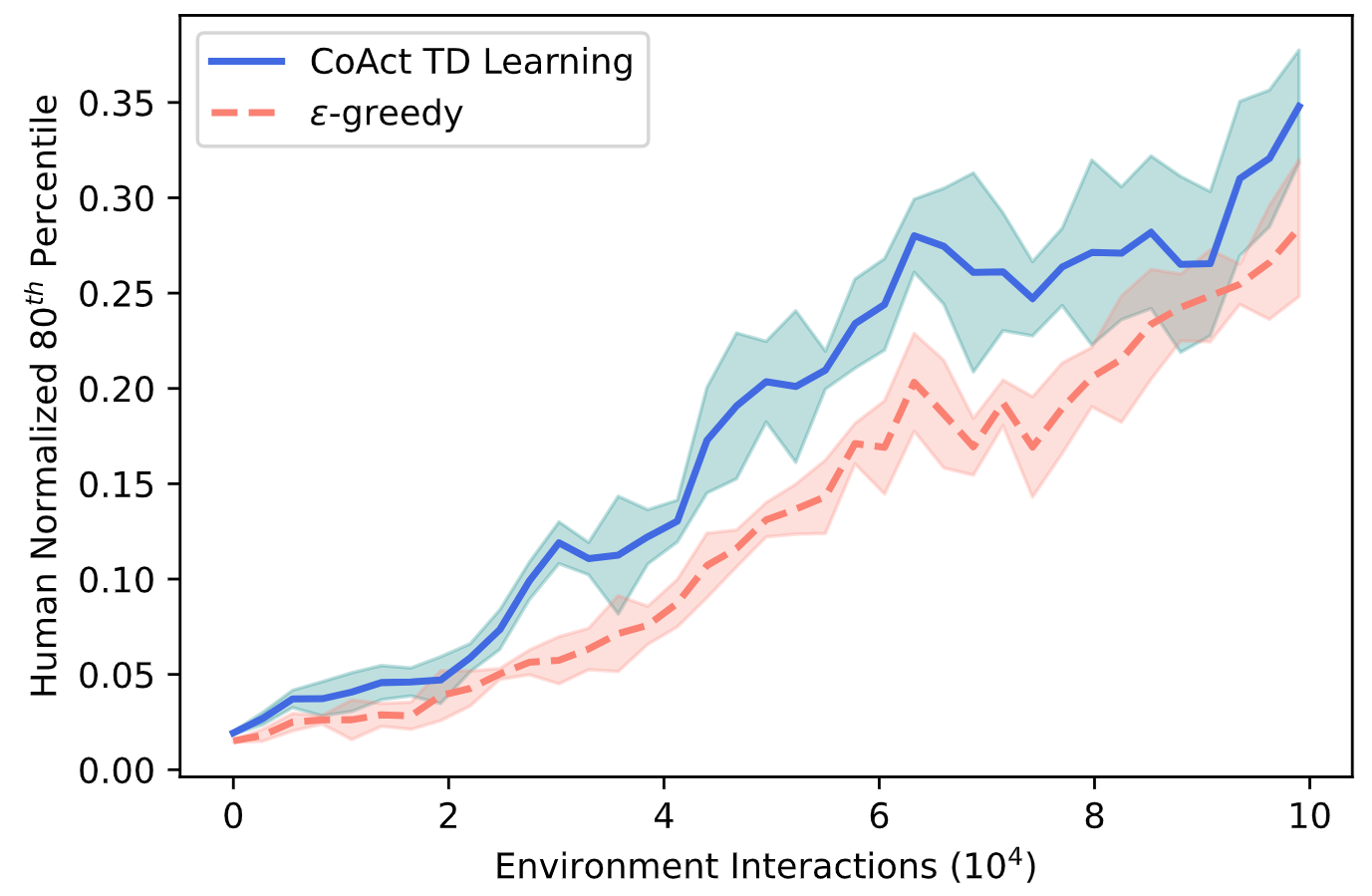}}{}
\end{center}
\vskip -0.225in
\caption{Human normalized scores median and 80$^{\textrm{th}}$ percentile over all games in the Arcade Learning Environment (ALE) 100K benchmark in DDQN for CoAct TD Learning algorithm and the canonical temporal difference learning with $\epsilon$-greedy. Left:Median. Right: 80$^{\textrm{th}}$ Percentile.}
\vskip -0.09in
\label{deltatdfigure}
\end{figure*}
\begin{figure*}[t!]
\footnotesize
\begin{center}
\vskip -0.1in
\stackunder[3pt]{\includegraphics[scale=0.19]{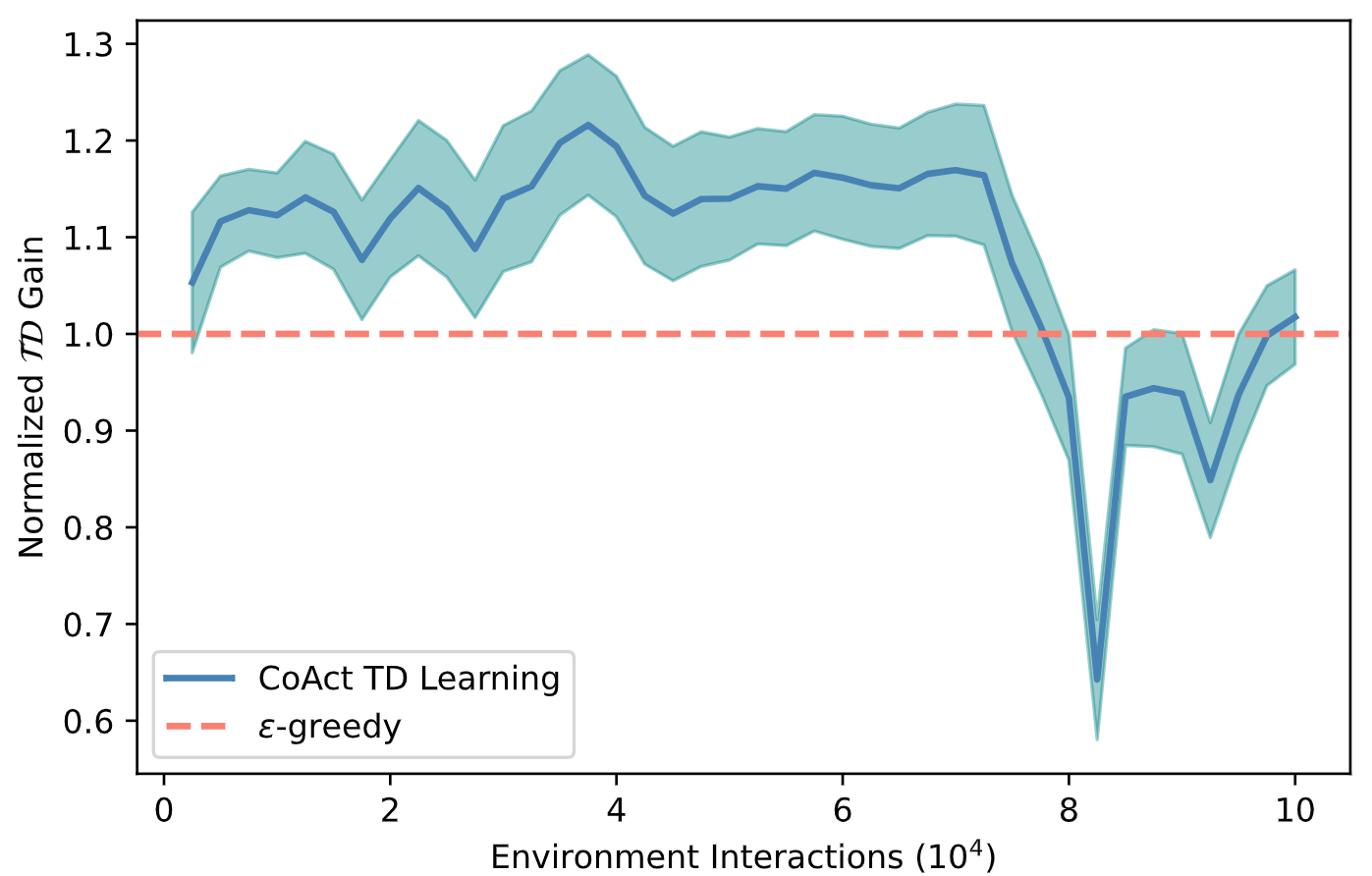}}{\scriptsize{}}
\stackunder[3pt]{\includegraphics[scale=0.18]{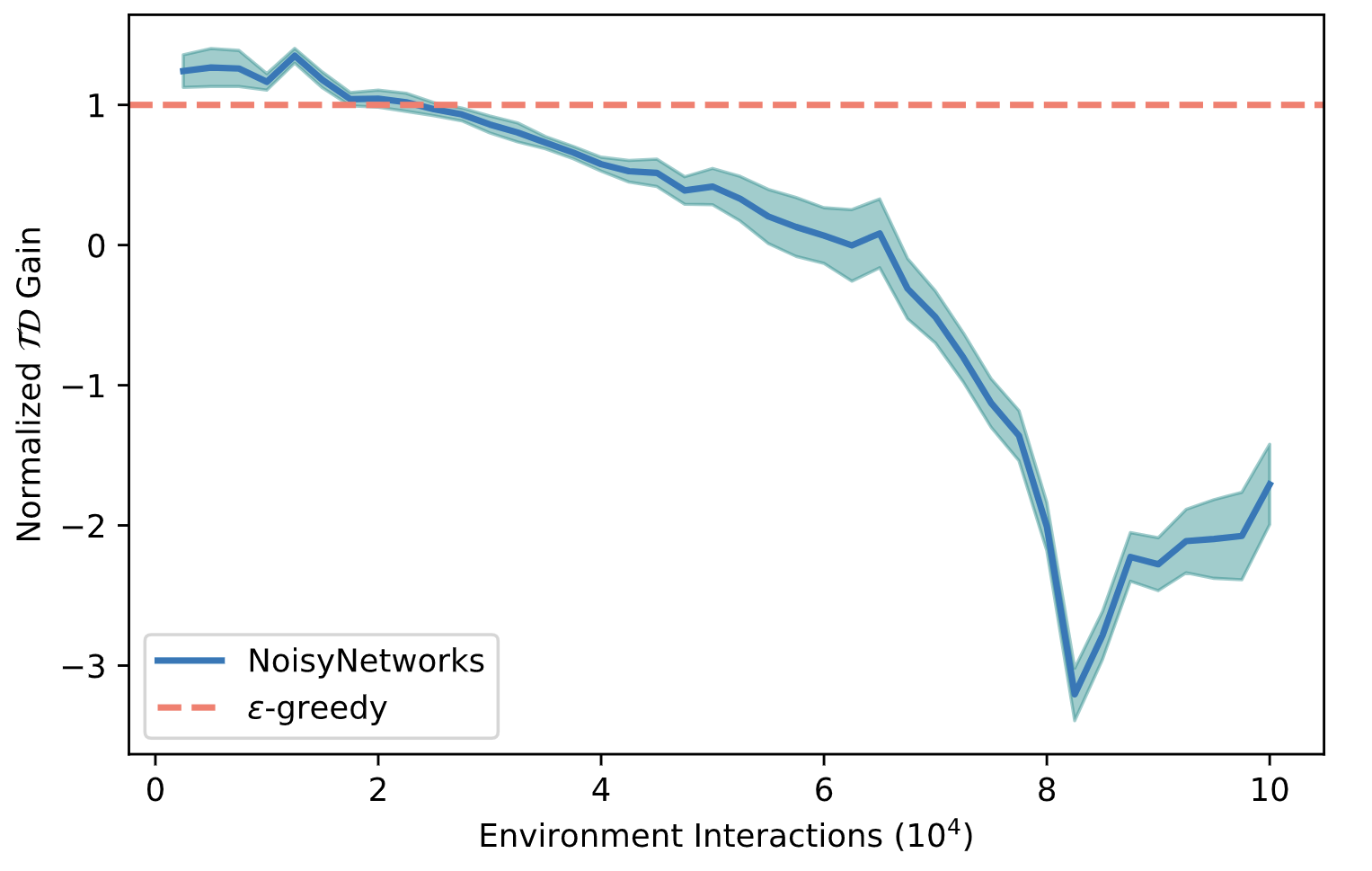}}{\scriptsize{}}
\stackunder[3pt]{\includegraphics[scale=0.192]{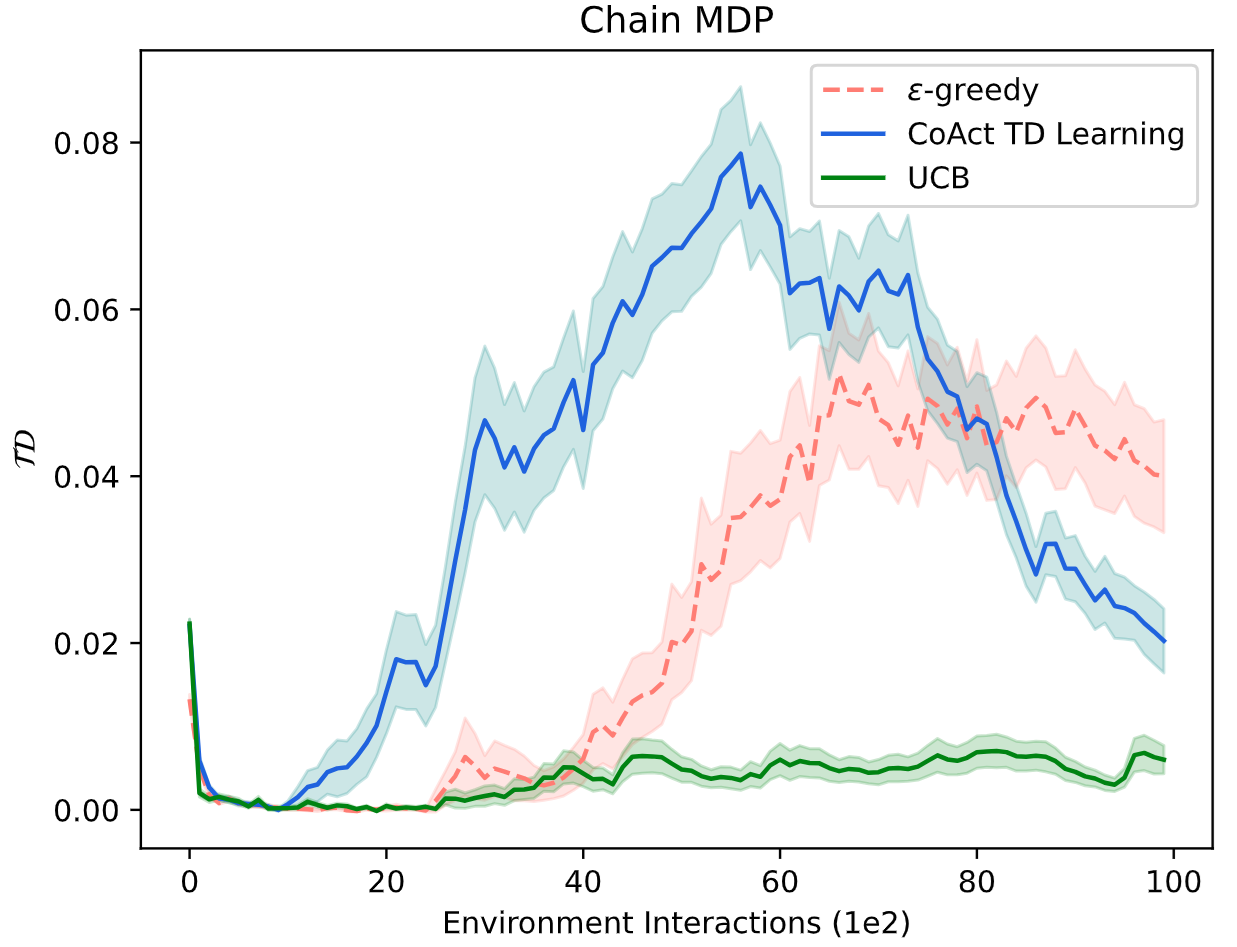}}{\scriptsize{}}
\end{center}
\vskip -0.19in
\caption{Left and Middle: Normalized temporal difference $\mathcal{TD}$ gain median across all games in the Arcade Learning Environment 100K benchmark for CoAct TD Learning and NoisyNetworks. Right: Temporal difference $\mathcal{TD}$ when exploring chain MDP with Upper Confidence Bound (UCB) method, $\epsilon$-greedy and our proposed algorithm CoAct TD Learning.}
\label{tdgain}
\vskip -0.14in
\end{figure*}

\section{Investigating the Temporal Difference}
\label{temporal}
In Section \ref{worstq} we provided the theoretical analysis and justification for collecting experiences with counteractive actions, i.e. the minimum $Q$-value action, in which counteractive actions increase the temporal difference.
Increasing temporal difference of the experiences results in novel transitions, and hence accelerates learning \citep{david97}.
The theoretical analysis from Theorem \ref{tdprop} and Theorem \ref{doubleqprop} shows that taking the minimum value action results in an increase in the temporal difference. 
In this section, we further investigate the temporal difference and provide empirical measurements of the TD.
To measure the change in the temporal difference when taking the minimum action versus the average action, we compare the temporal difference obtained by CoAct TD Learning with that obtained by several other canonical methods.
In more detail, during training, for each batch $\Lambda$ of transitions of the form $(s_t,a_t,s_{t+1})$ we record,
the temporal difference $\mathcal{TD}$
\begin{align*}
\mathbb{E}_{(s_t,a_t,s_{t+1}) \sim \Lambda}  \mathcal{TD}(s_t,a_t,s_{t+1}) 
= \mathbb{E}_{(s_t,a_t,s_{t+1}) \sim \Lambda} [r(s_t,a_t) 
+ \gamma\max_a Q_{\theta}(s_{t+1},a) - Q_{\theta}(s_t,a_t)].
\end{align*}
The results reported in Figure \ref{TDgames} and Figure \ref{tdgain} further confirm the theoretical predictions made via Definition \ref{tderrdef}, Theorem \ref{doubleqprop} and Theorem \ref{tdprop}.
In addition to the results for individual games reported in Figure \ref{TDgames}, we compute a normalized measure of the gain in temporal difference achieved when using CoAct TD Learning and plot the median across games. We define the normalized $\mathcal{TD}$ gain to be, 
$\textrm{Normalized}\:\mathcal{TD}\: \textrm{Gain} = 1 + (\mathcal{TD}_{\textrm{method}} - \mathcal{TD}_{\epsilon\textrm{-greedy}})/(\lvert \mathcal{TD}_{\epsilon\textrm{-greedy}} \rvert)$,
where $\mathcal{TD}_{\textrm{method}}$ and $\mathcal{TD}_{\epsilon\textrm{-greedy}}$ are the temporal difference for any given learning method and $\epsilon$-greedy respectively.
The leftmost and middle plot of Figure \ref{tdgain} report the median across all games of the normalized $\mathcal{TD}$ gain results for CoAct TD Learning and NoisyNetworks in the Arcade Learning Environment 100K benchmark. Note that, consistent with the predictions of Theorem \ref{tdprop}, the median normalized temporal difference gain for CoAct TD Learning is up to 25 percent larger than that of $\epsilon$-greedy. The results for NoisyNetworks demonstrate that alternate experience collection methods lack this positive bias relative to the uniform random action.
Further note that to guarantee that every action has non-zero probability of being chosen in every possible state for guaranteeing convergence of $Q$-learning, one can additionally introduce noise, and achieve higher temporal difference by CoAct TD learning. Albeit, across all of the experiments introduction of the noise was not necessary, as there is sufficient noise in the training process that satisfies the property for convergence. Hence, across all the benchmarks CoAct-TD Learning results in consistently and substantially better performance.
\textbf{CoAct TD Learning} is extremely modular, \textbf{only requires two lines of additional code} and can be used as \textbf{a drop-in replacement for any baseline algorithm} that uses the canonical methods. The fact that, as demonstrated in Table \ref{exptable}, CoAct TD Learning significantly outperforms noisy networks in the low-data regime is further evidence of the advantage the positive bias in temporal difference confers. The rightmost plot of Figure \ref{tdgain} reports $\mathcal{TD}$ for the motivating example of the chain MDP. As in the large-scale experiments, prior to convergence CoAct TD Learning exhibits a notably larger temporal difference relative to the canonical baseline methods, thus CoAct TD Learning achieves accelerated learning across domains from tabular MDPs to large-scale.

\vskip-0.5in

\section{Conclusion}
In our study we focus on the following questions: 
\textit{(i) Is it possible to maximize sample efficiency in deep reinforcement learning without additional computational complexity by solely rethinking the core principles of learning?,
(ii) What is the foundation and theoretical motivation of the learning paradigm we introduce 
that results in one of the most computationally efficient ways to explore in deep reinforcement learning?} and, 
\textit{(iii) How would the theoretically well-motivated approach transfer to high-dimensional complex MDPs?} 
To be able to answer these questions we propose a novel, theoretically motivated method with zero additional computational cost based on counteractive actions that minimize the state-action value function in deep reinforcement learning.
We demonstrate theoretically that our method CoAct TD Learning based on minimization of the state-action value results in higher temporal difference, and thus creates novel transitions with more unique experience collection.
Following the theoretical motivation we initially demonstrate in a motivating example in the chain MDP setup that our proposed method CoAct TD Learning results in achieving higher sample-efficiency. 
Then, we expand this intuition and conduct large scale experiments in the Arcade Learning Environment, and demonstrate that our proposed method CoAct TD Learning increases the performance on the Arcade Learning Environment 100K benchmark by $248\%$.

\bibliography{counteractiveRLbib}

\bibliographystyle{abbrvnat}

\end{document}

%% file: math_commands.tex
\usepackage{amsmath,amsfonts,bm}

\def\eqref#1{equation~\ref{#1}}
\def\1{\bm{1}}

\DeclareMathAlphabet{\mathsfit}{\encodingdefault}{\sfdefault}{m}{sl}
\SetMathAlphabet{\mathsfit}{bold}{\encodingdefault}{\sfdefault}{bx}{n}

\newcommand{\R}{\mathbb{R}}

\DeclareMathOperator*{\argmax}{arg\,max}
\DeclareMathOperator*{\argmin}{arg\,min}